\documentclass{article}

\usepackage{fullpage}
\usepackage{authblk}

\usepackage{array}
\usepackage{tabularx}
\usepackage{booktabs}
\usepackage{ragged2e}
\usepackage{microtype}
\usepackage[breaklinks=true,pdfborder={0 0 0}, colorlinks=true, citecolor=blue]{hyperref}
\usepackage{graphicx}

\usepackage[utf8]{inputenc} 
\usepackage[T1]{fontenc}    
\usepackage{hyperref}       
\usepackage{xcolor}       
\usepackage{url}            
\usepackage{amsfonts}       
\usepackage{nicefrac}       
\usepackage{algorithm}
\usepackage[compatible]{algpseudocode}
\renewcommand{\algorithmiccomment}[1]{\bgroup\hfill$\triangleright$~#1\egroup}

\usepackage{amsmath,amsthm, mathtools, amssymb, bbm, varwidth, bm}

\usepackage[capitalize,noabbrev]{cleveref}

\theoremstyle{plain}
\newtheorem{theorem}{Theorem}[section]

\newtheorem{lemma}[theorem]{Lemma}
\newtheorem{fact}[theorem]{Fact}
\newtheorem{corollary}[theorem]{Corollary}
\theoremstyle{definition}
\newtheorem{definition}[theorem]{Definition}

\theoremstyle{remark}
\newtheorem{remark}[theorem]{Remark}

\usepackage[textsize=tiny]{todonotes}

\def\prev{\textnormal{prev}}
\def\regret{\textnormal{Regret}_{\gamma}}
\def\bhsh{\mathcal{B}_{\textnormal{HS}}(\mathcal{H})}
\def\ihi{I_{\mathcal{H}, *}}
\def\ihsi{I_{\textnormal{HS}, *}}
\def\dpse{d_{\textnormal{pse}}}
\def\deff{d_{\textnormal{eff}}}

\def\kkh{\mathcal{K}_{(s_t, a_t)}}
\def\kth{\mathcal{K}_{(s_{\tau}, a_{\tau})}}
\def\skth{\mathcal{K}_{\tau}}
\def\kxi{\mathcal{K}_{x_i}}
\def\kxj{\mathcal{K}_{x_j}}
\def\kxt{\mathcal{K}_{x_t}}
\def\kxn{\mathcal{K}_{x_n}}
\def\wh{\widehat{W}}
\def\vt{\widetilde{V}}
\def\qt{\widetilde{Q}}
\def\qh{\widehat{Q}}
\def\gammah{\widehat{\Gamma}}

\DeclarePairedDelimiter\abs{\lvert}{\rvert}
\DeclarePairedDelimiter\norm{\lVert}{\rVert}
\DeclarePairedDelimiter\hnorm{\lVert}{\rVert_{\mathcal{H}}}
\DeclarePairedDelimiter\hsnorm{\lVert}{\rVert_{\textnormal{HS}}}
\DeclarePairedDelimiter\hinner{\langle}{\rangle_{\mathcal{H}}}
\DeclarePairedDelimiter\hsinner{\langle}{\rangle_{\textnormal{HS}}}

\DeclareMathOperator*{\argmax}{arg\,max\,}

\DeclareMathOperator*{\tr}{tr}

\DeclareMathOperator*{\expt}{\mathbb{E}}

\DeclareMathOperator*{\clip}{\textnormal{clip}}

\makeatletter
\newcommand{\vast}{\bBigg@{3.5}}
\newcommand{\Vast}{\bBigg@{5}}
\makeatother

\usepackage{natbib}
\setcitestyle{authoryear,round,citesep={;},aysep={,},yysep={;}}
\renewcommand{\cite}[1]{\citep{#1}}

\title{Provably Efficient Kernelized Q-Learning}
\author{Shuang Liu\thanks{s3liu@eng.ucsd.edu} }

\author{Hao Su\thanks{haosu@eng.ucsd.edu} }

\affil{
  University of California, San Diego
}
\begin{document}
\maketitle

\vskip 0.3in




\begin{abstract}
We propose and analyze a kernelized version of Q-learning. Although a kernel space is typically infinite-dimensional, extensive study has shown that generalization is only affected by the \emph{effective dimension} of the data. We incorporate such ideas into the Q-learning framework and derive regret bounds for arbitrary kernels. In particular, we provide concrete bounds for linear kernels and Gaussian RBF kernels; notably, the latter bound looks almost identical to the former, only that the actual dimension is replaced by a different notion of dimensionality. Finally, we test our algorithm on a suite of classic control tasks; remarkably, under the Gaussian RBF kernel, it achieves reasonably good performance after only 1000 environmental steps, while its neural network counterpart, deep Q-learning, still struggles.
\end{abstract}
\section{Introduction}
Q-learning, which dates back to as early as~\citet{watkins1992q}, has undergone tremendous development in the past decades and remains one of the most sample-efficient reinforcement learning (RL) frameworks in practice. The modern implementation of Q-learning maintains a value function $Q_t$ at each step $t$ such that
\begin{align}
\begin{aligned}
    &Q_t \approx \argmax_Q\\
    &\expt_{s, a, r, s'}\left[\left(r + \gamma \max_{a'}Q_{t - 1}(s', a') - Q(s, a)\right)^2\right],
\end{aligned}
    \label{eq:qlearning} 
\end{align}
where $(s, a, r, s')$, representing the quadruple of state, action, reward, next state, is sampled from historical data, $\gamma$ is a discounting factor, and actions that maximizes $Q_t$ are taken with priority~\cite{mnih2013playing}. The most successful implementation of the above paradigm is perhaps deep Q-learning (DQN)~\cite{mnih2013playing}, in which each $Q_t$ is represented by a (deep) neural network and actions are taken to maximize $Q_t$ with probability $1 - \epsilon$, and randomly otherwise.

Despite the success of DQN, our theoretical understanding of Q-learning is still lacking. While there exist regret analyses for certain versions of tabular Q-learning~\cite{jin2018q,liu2021regret}, these algorithms and their analyses hardly generalize to non-tabular settings. Algorithms and corresponding regret analyses have also been derived for more general function classes (see e.g.~\citet{jin2020provably,yang2020function,zhou21b}, among others). However, the proposed algorithms use either episodic value iteration~\cite{jin2020provably,yang2020function}, which is not sample-efficient in practice, or extended value iteration~\citet{zhou21b}, which is not time-efficient in practice.

We seek to fill the gap between theory and practice in this paper. Specifically, we propose and analyze a kernelized version of~\eqref{eq:qlearning}, named kernelized Q-learning (KQL), where each $Q_t$ is restricted to a reproducing kernel Hilbert space (RKHS) and optimized with kernel ridge regression, and exploration is done through upper confidence bound.

RKHS are non-parametric function classes that are rich enough in that they are dense in the space of continuous functions if the kernel is \emph{universal}~\cite{sriperumbudur2011universality}. Therefore, RKHS are, in some sense, as universal as neural networks~\cite{hornik1989multilayer} while being easier to analyze thanks to decades of research on kernel machines. Combining standard machinery for analyzing kernel-based learning with our novel approach to analyze Q-learning, we arrive at general regret bounds for arbitrary kernels (Theorem~\ref{thm:main}); in particular, we provide concrete regret bounds for linear kernels (Corollary~\ref{cor:linear}) and Gaussian RBF kernels (Corollary~\ref{cor:gaussian}). Notably, the latter looks almost identical to the former, only that the actual dimension is replaced by a different dimensionality that is at most polylogarithmic in the number of steps.

We complement our theoretical analyses with experiments on a subset of classic control tasks provided in OpenAI Gym~\cite{brockman2016openai}. We faithfully implement the exact KQL algorithm for which we derive regret bounds and choose hyperparameters based on the bounds. We demonstrate superior sample efficiency of KQL, even when compared with DQN. To the best of our knowledge, KQL is the first provably low-regret algorithm that excels in commonly used benchmarks. 

\section{Related Work}
\subsection{Reinforcement Learning}
Earlier regret analysis for RL focused on the tabular setting. \citet{jaksch2010near, osband2016lower} gave regret bounds in the average reward setting, \citet{osband2016generalization,azar2017minimax,jin2018q,simchowitz2019non,zhang2019regret,russo2019worst,zanette2019tighter,zhang2020almost,zhang2021reinforcement} gave regret bounds in the episodic setting. Regret bounds for the discounted setting was not studied until recently~\cite{liu2021regret,he2021nearly}, due to trickier definition of the regret.

Later work generalizes tabular analyses to the linear model episodic setting~\cite{wang2019optimism,jin2020provably,zanette2020frequentist}, the linear mixture model episodic setting~\cite{cai2020provably,ayoub2020model,modi2020sample,yang2020reinforcement,zhou21a,he2021logarithmic}, and the linear mixture model discounted setting~\cite{zhou21b}. We note here that the linear mixture model assumes the model class has finite (and small) degrees of freedom, which mitigates many challenges in model estimation (see discussion in \citet[Related Work]{jin2020provably}). In particular, while~\citet{zhou21b} gives a seemingly better regret bound than ours when the model class is linear in certain sense, they are under far stricter assumptions.

This paper is most related to~\citet{yang2020function}, which proposes and analyzes KOVI, a kernelized version of episodic value iteration. The major difference is that~\citet{yang2020function} operates on an episodic setting, where in each episode, an optimistically optimal Q-value function is calculated based on all the historical data using episodic value iteration. On the other hand, we focus on the non-episodic setting, where in each step, a new Q-value function, which is not necessarily (optimistically) optimal based on current historical data, is calculated from the Q-value function from the previous step using one-step discounting. Such difference indeed affects theoretical analysis in that the non-episodic update introduces pathological dependencies that must be properly handled (e.g., Lemma~\ref{lem:sum-of-ucb-bound}). Furthermore, it also makes our algorithm significantly more applicable to real-world problems. For example, in our \textsc{MountainCar} experiment (Section~\ref{sec:exp}), the length of an episode is 200. Given a budget of 1000 interactions, KOVI can only update its value functions $1000 / 200 - 1 = 4$ times (since each step in an episode uses a different value function). It is very unlikely that any progress can be made after such a small number of parameter updates (indeed, we were not able to make KOVI work on \textsc{Mountain-Car} after extensive parameter tuning or even hacking). On the other hand, KQL does not distinguish between different episodes, and can update its value function 999 times. In fact, we will show that KQL even outperforms DQN in the low-budget setting.

Other kernelization efforts include~\citet{xu2005kernel}, which uses kernel regression to approximate the discounted return of a Markov chain,~\citet{xu2007kernel}, which uses kernel regression as a subroutine in a policy iteration procedure, no exploration is involved thus no regret bound can be obtained;
~\citet{chowdhury2019online}, which requires very restrictive assumptions, as discussed in~\citet{yang2020reinforcement}; \citet{yang2020reinforcement}, which assumes linear mixture models, therefore suffers from aforementioned limitations of such models; 
\citet{ormoneit2002kernel,barreto2016practical}, which are based on local averaging and require Lipschitzness assumptions, and no exploration is involved thus no regret bound can be obtained;~\citet{domingues2021kernel}, which is similarly based on local averaging and require Lipschitzness assumption, but with Lipschitzness-based exploration.

More general function approximation classes have also been considered. Notably,~\citet{wang2020reinforcement,ayoub2020model} gave regret bounds in terms of Eluder dimension and covering number of the function class. However, most of these attempts either result in generally intractable algorithms
~\cite{krishnamurthy2016pac,jiang2017contextual,dann2018oracle,dong2020root, wang2020reinforcement, ayoub2020model}, or make very restrictive assumptions such as deterministic environment~\cite{wen2013efficient,wen2017efficient}, or the existence of a finite latent state space~\cite{du2019provably}.

\subsection{Bandits}
Bandit problems can be thought as a special case of RL problems by setting the horizon to $1$ in the episodic setting, or setting the discounting factor to $0$ in the discounted setting. Naturally, many ideas for RL analyses originates from the bandit literature.

The earliest model for bandit problems is multi-armed bandit~\cite{lai1985asymptotically}. Later work generalize the multi-armed model to the linear model~\cite{dani2008stochastic,rusmevichientong2010linearly,abbasi2011improved,li2021tight} and the linear contextual model~\cite{auer2002using,li2010contextual,chu2011contextual,li2019nearly}. Kernelization of bandit algorithms are further developed for kernel bandit~\cite{srinivas2009gaussian,srinivas2012information,chowdhury2017kernelized} and kernel contextual bandit~\cite{krause2011contextual,valko2013finite,zhou2020neural}.

Analyses for RL is generally much harder than for bandits, in that it needs to additionally avoid exponential dependencies on the state space and horizon. In particular, (near) optimal dependencies for these quantities require very sophisticated techniques.

\section{Preliminaries}
For a measurable space $X$, let $\mathcal{B}(X)$ be the space of real-valued bounded measurable functions over $X$ equipped with the supremum norm $\norm{\cdot}_{\infty}$, let $\mathcal{P}(X)$ be the set of all probability measures over $X$. For any positive integer $n$, denote by $[n]$ the set $\{1, 2, \cdots, n\}$. For any two functions $f, g: X\to[0, \infty)$, $f(x) = \mathcal{O}(g)$ means that there exists $c \geq 0$ such that for any $x\in X$, $f(x) \leq c \cdot g(x)$.

Let $\mathcal{H}$ be a Hilbert space. Denote by $\hinner{\cdot, \cdot}$ the inner product on $\mathcal{H}$. Denote by $\hnorm{\cdot}$ the norm induced by the inner product. If $T: \mathcal{H}\to\mathcal{H}$ is a bounded self-adjoint positive-definite linear operator, define the Mahalanobis norm
$
    \norm*{f}_T = \sqrt{\hinner{Tf, f}} 
    $.
Denote by $\bhsh$ the space of Hilbert-Schmidt operators from $\mathcal{H}$ to $\mathcal{H}$. Denote by $\hsinner{\cdot, \cdot}$ the inner product on $\bhsh$. Denote by $\hsnorm{\cdot}$ the norm induced by the inner product. 
For any $u, v\in\mathcal{H}$, denote by $u\otimes v$ the linear operator such that for any $f\in\mathcal{H}$, $(u\otimes v)f = \hinner{u, f}\cdot v$.

Let $\mathcal{H}$ be a real-valued RKHS over a set $X$ such that the corresponding kernel $\mathcal{K}$ is bounded. For any $x\in X$, denote by $\mathcal{K}_x$ the reproducing function at $x$, i.e., $\mathcal{K}_x$ is the unique function that satisfies for any $f\in\mathcal{H}$,
$
    f(x) = \hinner{f, \mathcal{K}_x} 
    $.
In fact, $\mathcal{K}_x = (x\mapsto\mathcal{K}(x, \cdot))$.
For any $f\in\mathcal{H}$, define the norm
\begin{align*}
    \norm{f}_* = \sup_{x\in X}\hinner{f, \mathcal{K}_x}.
\end{align*}
For any bounded linear operator $T:\mathcal{H}\to\mathcal{H}$, define the norm
\begin{align*}
    \norm{T}_* = \sup_{x\in X}\hinner{T\mathcal{K}_x, \mathcal{K}_x}.
\end{align*}

\subsection{Effective Dimension and Pseudo Dimension}
We are going to introduce two concepts used in the learning literature to capture the properties of a RKHS, \emph{effective dimension} and \emph{pseudo dimension}. Note that many similar concepts have been introduced (see \citet{srinivas2009gaussian,srinivas2012information,valko2013finite,chowdhury2017kernelized,yang2020reinforcement,yang2020function} among many others) under various names such as \emph{information gain}. All these variants are in some sense equivalent to each other, in that they all capture the effective dimensionality of a RKHS, perhaps up to a logarithmic factor.

In this section, let $X$ be a set, $\mathcal{K}: X\times X\to\mathbb{R}$ be a kernel, $n$ be a positive integer, $x_{1:n} = (x_1, x_2, \cdots, x_n)\in X^n$, and $\lambda > 0$, $\mathfrak{K}$ be a $n\times n$ matrix where $\mathfrak{K}[i][j] = \mathcal{K}(x_i, x_j)$,
$
    \Sigma = \sum_{i = 1}^n \kxi\otimes\kxi
    $.

\subsubsection{Effective Dimension}
\label{sec:effective-dim}
\begin{definition}[Effective dimension~\cite{zhang2005learning,HastieTF09,calandriello2017second}\footnote{Also called effective degrees of freedom in~\citet{HastieTF09}.}]
\label{def:effective-dimension}
The effective dimension of $x_{1:n}$ w.r.t. $\mathcal{K}$ at scale $\lambda$ is defined to be
\begin{align*}
    \deff(\lambda, x_{1:n}) = \tr\left((\mathfrak{K} + \lambda I)^{-1}\mathfrak{K}\right).
\end{align*}
We also define $\deff(\lambda, \emptyset) = 0$.
\end{definition}
\begin{remark}
It is easy to see that $\deff(\lambda, x_{1:n})$ is a non-negative non-increasing function of $\lambda$, it tends to $0$ as $\lambda\to\infty$. It is also a bounded function of $\lambda$ in that $\deff(\lambda, x_{1:n}) \leq n$.
\end{remark}

The following lemma gives a bases-independent representation of $\deff$; in particular, it shows that the effective dimension is invariant to the permutation of data, therefore $x_{1:n}$ in the definition of $\deff$ can be simply treated as a set.
\begin{lemma}[\cite{zhang2005learning}]
\label{lem:deff-equivalent}
\begin{align*}
\deff(\lambda, x_{1:n}) = \tr \left(\left(\Sigma + \lambda I\right)^{-1}\Sigma\right).
\end{align*}
\end{lemma}

The following lemma gives yet another representation of $\deff$, which will be particularly useful when analyzing our proposed algorithm later.
\begin{lemma}
\label{lem:deff-identity}
$
    \deff(\lambda, x_{1:n}) = \sum_{i = 1}^n\norm*{\kxi}^2_{\left(\Sigma + \lambda I\right)^{-1}}
    $.
\end{lemma}
The following lemma shows that the effective dimension is non-decreasing as data accumulate.
\begin{lemma}
\label{lem:deff-non-decreasing}
$
    \deff(\lambda, x_{1:(n - 1)})\leq \deff(\lambda, x_{1:n}) 
    $.
\end{lemma}

It is instrumental to see how $\deff$ behaves under common kernels. For linear kernels, the following lemma is well-known.
\begin{lemma}[Effective dimension under linear kernels]
\label{lem:deff-linear}
If $X$ is a $d$-dimensional Euclidean space and
$
    \mathcal{K}(x, y) = x^{\intercal}y 
    $,
then $\deff(\lambda, x_{1:n}) \leq d$.
\end{lemma}
For Gaussian RBF kernels, we have the following recent result. Notably, the effective dimension grows only polylogarithmically with $\frac{n}{\lambda}$.
\begin{lemma}[Effective dimension under Gaussian RBF kernels~\cite{altschuler2018massively}]
\label{lem:deff-gaussian}
Let $d$ be a positive integer, $X = \left\{x\in\mathbb{R}^d: \norm{x}_2\leq 1\right\}$, and 
$
    \mathcal{K}(x, y) = e^{-\eta\norm{x - y}_2^2}
    $
for some $\eta \geq 0$, then for any $\lambda\leq n$,
\begin{align*}
\deff(\lambda, x_{1:n}) \leq 3\left(6 + \frac{41}{d}\eta + \frac{3}{d}\ln\frac{n}{\lambda}\right)^d.
\end{align*}
\end{lemma}

\subsubsection{Pseudo Dimension}
\label{sec:pseudo-dim}
\begin{definition}[Pseudo dimension\footnote{This quantity was not given a name in the original paper, we name it in light of Definition~\ref{def:effective-dimension} and Lemma~\ref{lem:effective-dimension-bound} to facilitate the discussion.}~\cite{jezequel2019efficient}]
\label{def:pseudo-dimension}
The pseudo dimension of $x_{1:n}$ w.r.t. $\mathcal{K}$ at scale $\lambda$ is defined to be
\begin{align*}
\dpse(\lambda, x_{1:n}) = \ln\left(\det\left(I + \frac{\mathfrak{K}}{\lambda}\right)\right)
\end{align*}
We also define $\dpse(\lambda, \emptyset) = 0$.
\end{definition}
\begin{remark}
Similar to $\deff$, $\dpse(\lambda, x_{1:n})$ is also a non-negative non-increasing function of $\lambda$, it tends to $0$ as $\lambda\to\infty$. However, in contrast to $\deff$, $\dpse(\lambda, x_{1:n})$ is not a bounded function of $\lambda$; if $\mathfrak{K}$ is not a zero matrix, as $\lambda\to 0$, it tends to infinity.
\end{remark}
\begin{remark}
Because the determinant only changes sign when swapping rows or columns, the pseudo dimension is invariant to the permutation of data, therefore just like in the case of $\deff$, $x_{1:n}$ in the above definition can be simply treated as a set.
\end{remark}

The following lemma shows that $\dpse$ is at most a logarithmic factor (in terms of $\frac{n}{\lambda}$) larger than $\deff$.
\begin{lemma}[\cite{jezequel2019efficient}]
\label{lem:effective-dimension-bound}
\begin{align*}
\dpse(\lambda, x_{1:n}) \leq \ln\frac{e(n + \lambda)}{\lambda}\cdot\deff(\lambda, x_{1:n}).
\end{align*}
\end{lemma}

The following lemma will play an important role in our analysis.
\begin{lemma}
\label{lem:sum-of-ucb-bound}
The following are true:
\begin{enumerate}
    \item[\textnormal{(I).}] $\dpse(\lambda, x_{1:(n - 1)}) \leq \dpse(\lambda, x_{1:n})$.
    \item[\textnormal{(II).}]
Let 
$
    \Sigma_i = \sum_{j = 1}^i \kxj\otimes\kxj 
    $,
then for any $\gamma\in[0, 1)$,
\begin{align*}
    &\sum_{i = 1}^n\sum_{\tau = 1}^{i}\gamma^{i - \tau}\norm*{\kxi}^2_{\left(\Sigma_{\tau-1} + \lambda I\right)^{-1}} \\
    &\leq \frac{1/\lambda}{\ln\left(1 + 1/\lambda\right)(1 - \gamma)^2}\cdot
    \dpse(\lambda, x_{1:n}).
\end{align*}
\end{enumerate}
\end{lemma}
In particular, Lemma~\ref{lem:sum-of-ucb-bound}.(I) says that, similar to the effective dimension (Lemma~\ref{lem:deff-non-decreasing}), the pseudo dimension is also non-decreasing as data accumulate.

To conclude the introduction of pseudo dimension, let us introduce a generalization of ~\citet[Theorem 1]{abbasi2011improved} stated in terms of $\dpse$ in the following Lemma. Its proof is almost identical to the proof in the special case when $\mathcal{H}$ has a linear kernel. We also note that the inequality in the lemma visually resembles Theorem 1 in~\citet{chowdhury2017kernelized}, however, the quantities on the left hand side of the inequalities are actually quite different.
\begin{lemma}[Self-normalized bound for $\mathcal{H}$-valued martingales]
\label{lem:self-normalized}
Let $\left\{F_t\right\}_{t = 0}^{\infty}$ be a filtration,  $\left\{\eta_{t}\right\}_{t = 1}^{\infty}$ be a real-valued stochastic process such that $\eta_{t}$ is $F_t$ measurable, and is zero mean and $R$-sub-Gaussian conditioned on $F_{t - 1}$. Let $\left\{x_t\right\}_{t = 1}^{\infty}$ be a $X$-valued stochastic process such that $x_t$ is $F_{t - 1}$ measurable. Then for any $p > 0$, with probability at least $1 - p$, for all $T \geq 0$,
\begin{align*}
    &\norm*{\sum_{t = 1}^{T}\eta_{t}\kxt}^2_{\left(\lambda I + \sum_{\tau = 1}^{T} \mathcal{K}_{x_{\tau}}\otimes\mathcal{K}_{x_{\tau}}\right)^{-1}} \\
    &\leq 2R^2\left(\dpse(\lambda, x_{1:T}) + \ln\left(\frac{1}{p}\right)\right).
\end{align*}
\end{lemma}

\subsection{Covering Number of Operators}
Let $T: X\to Y$ be a bounded linear operator where $X$ and $Y$ are normed vector spaces.
$\mathcal{N}(\epsilon, T)$, the $\epsilon$-covering number of $T$, is defined to be the cardinality of the smallest set $V\subseteq Y$ such that for any $x\in X$, $\norm{x}\leq 1$, there exists a $v\in V$ such that
$
   \norm{v - Tx}\leq \epsilon 
   $.
If there is no such set $V$ of finite cardinality, then $N(\epsilon, T)$ is defined to be $\infty$. 

Given a real-valued RKHS $\mathcal{H}$, we are interested in the covering numbers of two identity mappings, 
\begin{align*}
\ihi&: (\mathcal{H}, \hnorm{\cdot})\to(\mathcal{H}, \norm{\cdot}_*)\\
\ihsi&: (\bhsh, \hsnorm{\cdot})\to(\bhsh, \norm{\cdot}_*).
\end{align*}
Both $\mathcal{N}(\epsilon, \ihi)$ and $\mathcal{N}(\epsilon, \ihsi)$ captures properties $\mathcal{H}$. However, unlike the effective dimension introduced in Section~\ref{sec:effective-dim} and the pseudo dimension introduced in Section~\ref{sec:pseudo-dim} that depend on a scale $\lambda$, these covering numbers are data-independent and depend on a granularity $\epsilon$ instead.  
\begin{lemma}[$\mathcal{N}(\epsilon, \ihi)$ and $\mathcal{N}(\epsilon, \ihsi)$ under linear kernel]
\label{lem:covering-linear}
If $X = \left\{x\in\mathbb{R}^d: \norm{x}_2\leq 1\right\}$ and $\mathcal{K}(x, y) = x^{\intercal}y$, then
\begin{align*}
    \ln \mathcal{N}(\epsilon, \ihi) &\leq d\ln\left(1 + 2/\epsilon\right)\\ 
    \ln \mathcal{N}(\epsilon, \ihsi) &\leq d^2\ln\left(1 + 2/\epsilon\right).
\end{align*}
\end{lemma}

\begin{lemma}[$\mathcal{N}(\epsilon, \ihi)$ under Gaussian RBF kernel~\cite{kuhn2011covering}]
\label{lem:covering-gaussian-ihi}
If $X = \left\{x\in\mathbb{R}^d: \norm{x}_2\leq 1\right\}$ and $\mathcal{K}(x, y) = e^{-\eta\norm{x - y}_2^2}$ for some $\eta \geq 0$, then \begin{align*}
    \ln \mathcal{N}(\epsilon, \ihi) \leq \left\lceil2\left(\ln\frac{2}{\epsilon}+e^2\eta\right)\right\rceil^d\ln\left(1 + \frac{4}{\epsilon}\right)
\end{align*}
\end{lemma}

\begin{lemma}[$\mathcal{N}(\epsilon, \ihsi)$ under Gaussian RBF kernel]
\label{lem:covering-gaussian-ihsi}
If $X = \left\{x\in\mathbb{R}^d: \norm{x}_2\leq 1\right\}$ and $\mathcal{K}(x, y) = e^{-\eta\norm{x - y}_2^2}$ for some $\eta \geq 0$, then \begin{align*}
    \ln\mathcal{N}(\epsilon, \ihsi) \leq \left\lceil 2\left(\ln\frac{2\sqrt{2}}{\epsilon}+e^2\eta \right)\right\rceil^{2d}\ln\left(1 + \frac{4}{\epsilon}\right).
\end{align*}
\end{lemma}

\section{Setup}
\label{sec:setup}
Consider a Markov decision process (MDP) with state space $\mathcal{S}$, action space $\mathcal{A}$, transition function 
$
    \mathbb{P}: \mathcal{S}\times\mathcal{A}\to \mathcal{P}(\mathcal{S}),
    $
and reward function\footnote{Our analysis easily generalizes to stochastic rewards, deterministic reward function is chosen for clarity.}
$
    r_*: \mathcal{S}\times\mathcal{A}\to[0, 1].
    $
Consider a $T$-step RL process. Denote by $s_t$ and $a_t$ the state and action taken in the $t$th step ($t$ starts from $1$), the algorithm receives reward $r_*(s_t, a_t)$ after taking action $a_t$ and a new state $s_{t + 1}$ is sampled from $\mathbb{P}(s_t, a_t)$. 

We measure the performance of an RL algorithm by its \emph{$\gamma$-regret}~\cite{liu2021regret,he2021nearly,zhou21b}, defined by
\begin{align}
    \regret(T) = (1 - \gamma)\sum_{t = 1}^{T}\left(V_*(s_t) - V_t\right),\label{eq:regret-def}
\end{align}
where $V_*(s)$ is the optimal $\gamma$-discounted return starting from state $s$, in the sense that
\begin{align}
    V_*(s) = \sup_{a\in\mathcal{A}}\mathbb{E}_{s'\sim\mathbb{P}(s, a)}\left[\left(r_*(s, a) + \gamma V_*(s')\right)\right], \label{eq:attain}
\end{align}
and 
\begin{align*}
   V_t = \sum_{\tau = t}^{\infty}\gamma^{\tau - t} r_*(s_{\tau}, a_{\tau}). 
\end{align*}
Here the multiplier $(1 - \gamma)$ in~\eqref{eq:regret-def} is only for normalization such that $\regret(T) \leq T$.
Denote by $a^*(s)$ any action that attains $V_*(s)$ in~\eqref{eq:attain}.

Let $\mathcal{H}$ be a real-valued RKHS defined on $\mathcal{S}\times\mathcal{A}$ such that the corresponding kernel $\mathcal{K}: (\mathcal{S}\times\mathcal{A})^2\to \mathbb{R}$ is bounded. Without loss of generality, we will assume $\norm{\mathcal{K}}_{\infty} > 0$; and we further assume $\mathcal{K}$ is properly normalized so that for any $(s, a)\in\mathcal{S}\times\mathcal{A}$, $\hnorm*{\mathcal{K}_{(s, a)}}\in[1/2, 1]$.

Let us convert $\mathbb{P}$ into a linear operator by defining
\begin{align*}
    M_*: \mathcal{B}(\mathcal{S})&\to\mathcal{B}(\mathcal{S}\times\mathcal{A})\\
    f&\mapsto \left( (s, a)\mapsto\mathbb{E}_{\mathbb{P}(s, a)}[f]\right).
\end{align*}
We will assume there exists an (unknown) $M: \mathcal{B}(\mathcal{S})\to\mathcal{H}$ such that  
\begin{align*}
    \sup_{\substack{f\in\mathcal{B}(\mathcal{S})\\\norm{f}_{\infty}\leq 1}}\norm*{(M -M_*)f}_{\infty} \leq \epsilon.
\end{align*}
The operator norm of $M$ is defined by
\begin{align*}
    \norm*{M} = \sup_{\substack{f\in \mathcal{B}(\mathcal{S})\\\norm{f}_{\infty}\leq 1}} \hnorm*{M(f)}.
\end{align*}
We assume
$
    \norm*{M}\leq \rho.
    $
Similarly, for the reward function, we will also assume there exists an (unknown) $r\in\mathcal{H}$ such that 
\begin{align*}
    \norm{r_* - r}_{\infty} \leq \epsilon, 
    ~~~\hnorm{r}\leq \rho.
\end{align*}

Let $\sigma\in[0, 1]$ be a bound on the stochasticity of the MDP. Specifically, for any $(s, a)\in\mathcal{S}\times\mathcal{A}$, $f\in \mathcal{B}(\mathcal{S})$ such that $\norm{f}_{\infty} \leq 1$, $f(s')$ is $\sigma$-sub-Gaussian when $s'\sim \mathbb{P}(s, a)$. Clearly, $\sigma$ can always be set to $1$, and for a deterministic MDP, $\sigma$ can be set to $0$.

\section{Algorithm}
We propose kernelized Q-learning (KQL) (Algorithm~\ref{alg:main}), which can run in time $\mathcal{O}((T + g)T^2 \abs{\mathcal{A}})$, where $g$ is the time to compute $x, y\mapsto\mathcal{K}(x, y)$. When the kernel is linear, the running time can be improved (by special implementation) to $\mathcal{O}(Td^2\abs{\mathcal{A}})$, where $d$ is the dimension of the feature space that $\mathcal{S}\times\mathcal{A}$ is embedded into.

\begin{algorithm*}[t]
\caption{Kernelized Q-Learning (KQL)}
\begin{algorithmic}[1]
\State \textrm{\textbf{Parameters:} number of steps $T$, discounting factor $\gamma$, kernel $\mathcal{K}$, $\lambda > 0$, $0 \leq \beta \leq \frac{2\sqrt{T + \lambda}}{1 - \gamma}$.}
\State{\textrm{\textbf{Initialize:} $\wh_1 = \frac{1}{\lambda} I$. $\qh_1 = 0$.}}
\vspace{0.5em}

\State{Receive the initial state $s_1$.}
\For{step $t = 1, 2, \cdots, T$}

    \For{$a\in\mathcal{A}$}
    \State{$\gammah_t[t][a] = \norm*{\mathcal{K}_{(s_t, a)}}_{\wh_t}$.}\Comment{$\mathcal{O}(T^2)$.}
    \State{$\qt_t[t][a] = \clip_{[0, 1/(1 - \gamma)]}\left(\qh_t(s_t, a) + \beta\cdot\gammah_t[t][a]\right)$}\Comment{$\mathcal{O}(T)$.}
    \EndFor
    \State{Take action $a_t = \argmax_{a}\qt_t[t][a]$ and observe $s_{t + 1}$.}\Comment{$\mathcal{O}(T^2\abs{\mathcal{A}}$).}
    
    \For{$\tau = 1, 2, \cdots, t - 1$}
    \For{$a\in\mathcal{A}$}
    \State{$\qt_t[\tau][a] = \clip_{[0, 1/(1 - \gamma)]}\left(\qh_t(s_{\tau}, a) + \beta\cdot\gammah_t[\tau][a]\right)$}\Comment{$\mathcal{O}(T)$.}
    \EndFor
    \EndFor
    
    \State{$u = \wh_t\kkh$.}\Comment{$\mathcal{O}(T^2)$.}
    \State{$s = \hinner*{u, \kkh}$.}\Comment{$\mathcal{O}(T^2)$.}
    \State{$\wh_{t + 1} = \wh_t - \frac{u\otimes u}{1 + s}$.}\Comment{$\mathcal{O}(T^2)$.}
    \For{$\tau = 1, 2, \cdots, t$}
    \For{$a\in\mathcal{A}$}
    \State{$\gammah_{t + 1}[\tau][a] = \sqrt{\left(\gammah_{t}[\tau][a]\right)^2 - \frac{u^2(s_{\tau}, a)}{1 + s}}$.}\label{line:pre}\Comment{$\mathcal{O}(T)$.}
    \EndFor
    \EndFor
    
    \State{$\qh_{t + 1} = \left(\sum_{\tau = 1}^{t}\left(r_*(s_{\tau}, a_{\tau}) + \gamma\max_a\qt_t[\tau + 1][a]\right)\cdot\mathcal{K}_{(s_{\tau}, a_{\tau})}\right)\wh_{t + 1}$.}\Comment{$\mathcal{O}(T^2\abs{\mathcal{A}})$.}\label{line:alg}
    
\EndFor
\end{algorithmic}
\label{alg:main}
\end{algorithm*}

We list closed-form representations of the variables maintained by Algorithm~\ref{alg:main} below. In fact, our main results in Section~\ref{sec:main-result} hold for any algorithm for which these variables are maintained.
\begin{fact}
\label{fact:conds}
Denote by $\skth = \kth$ and $r_{\tau} = r_*(s_{\tau}, a_{\tau})$. Algorithm~\ref{alg:main} satisfies
\vspace{-1em}
\begin{enumerate}
    \item[(1)] For any $t\in [T]$,
    $
        \wh_t =\left(\sum_{\tau = 1}^{t - 1}\skth\otimes\skth + \lambda I\right)^{-1}.
        $
    \item[(2)] For any $t\in[T]$,
    \begin{align*}
      \qh_t =\left(\sum_{\tau = 1}^{t - 1}\left(r_{\tau} + \gamma\max_a\qt_{t - 1}[\tau + 1][a]\right)\skth\right)\wh_t. 
    \end{align*}
    \item[(3)] For any $\tau, t\in[T]$ such that $\tau \leq t$, $a\in\mathcal{A}$
    \begin{align*}
      \qt_t[\tau][a] =\clip_{\left[0, \frac{1}{1 - \gamma}\right]}\left(\qh_t(s_{\tau}, a){+}\beta\norm*{\mathcal{K}_{(s_{\tau}, a)}}_{\wh_t}\right).  
    \end{align*}
\end{enumerate}
\end{fact}

\section{Main Results}
\label{sec:main-result}
We now state the main result of this paper, an upper bound on the regret for KQL.
\begin{theorem}
\label{thm:main}
Given $\lambda > 0$, $\gamma\in[0, 1)$, $\rho \geq 1 - \gamma$, $\epsilon \geq 0$, $\sigma \in [0, 1]$, $p > 0$, $d_\lambda\geq 1$, $c_{\lambda} \geq 0$, let
\begin{align*} 
&\beta = \frac{1}{1 - \gamma}\cdot\min\Bigg(2\sqrt{T + \lambda},~~3\rho\sqrt{\lambda} + \epsilon \sqrt{Td_{\lambda}}\\
&+2\sigma \sqrt{
    d_{\lambda}\ln\frac{e(T + \lambda)}{\lambda} +  \ln\frac{2}{p}+c_{\lambda}}\Bigg),
\end{align*}
in Algorithm~\ref{alg:main}, or any algorithm that satisfies (1)-(3) in Fact~\ref{fact:conds} and acts to maximize $\qt_t$ therein, with probability at least $1 - p$, if
\begin{align}
\begin{aligned}
d_{\lambda}&\geq \deff\left(\lambda, \{(s_t, a_t)\}_{i = 1}^T\right),\\
c_{\lambda}&\geq\ln\mathcal{N}\left(\frac{\lambda^2(1 - \gamma)}{4T^2}, \ihi\right)\\
&\ \ \ \ + \ln\mathcal{N}\left(\frac{\lambda^3(1 - \gamma)}{32(T + \lambda)^3}, \ihsi\right),
\end{aligned}\label{eq:betrue}
\end{align}
then
\begin{align*}
&\regret(T) =  \mathcal{O}\vast(\sqrt{\frac{Td_{\lambda}\log\frac{e(T + \lambda)}{\lambda}}{\log(1 + 1/\lambda)(1 - \gamma)^5}}~\cdot\\
&\left(\rho + \epsilon\sqrt{\frac{d_{\lambda}T}{\lambda}} + \sigma\sqrt{\frac{d_{\lambda}\log\frac{e(T + \lambda)}{\lambda p} + c_{\lambda}}{\lambda}}\right)\vast).
\end{align*}
\end{theorem}

Let us take a closer look at how the bound depends on $\rho$, $\epsilon$, $\sigma$, and $\lambda$. Note that the bound can be decomposed into three parts, involving $\rho$, $\epsilon$, $\sigma$ respectively: 

The first part, involving the complexity bound $\rho$, is
\begin{align*}
    \mathcal{O}\left(\frac{\rho \sqrt{T}}{(1 - \gamma)^{2.5}}\cdot \sqrt{\frac{d_{\lambda}\log\frac{e(K + \lambda)}{\lambda}}{\log(1 + 1/\lambda)}}\right).
\end{align*}
Here $d_{\lambda}$ upper bounds a non-increasing function of $\lambda$ that is bounded by $T$ and $\frac{\log \frac{e(T + \lambda)}{\lambda}}{\log(1 + 1/\lambda)}$ is a non-decreasing function of $\lambda$ that tends to $1$ as $\lambda\to 0$ and tends to $\infty$ as $\lambda\to\infty$. This suggests that $\lambda$ needs to strive a balance between the two conflicting dependencies.

The second part, involving the approximation bound $\epsilon$, is 
\begin{align*}
    \mathcal{O}\left(\frac{\epsilon  T}{(1 - \gamma)^{2.5}}\cdot d_{\lambda}\sqrt{\frac{\log\frac{e(T + \lambda)}{\lambda}}{\log(1 + 1/\lambda)\lambda}}\right).
\end{align*}
Here $d_{\lambda}$ upper bounds a non-increasing function of $\lambda$, $\frac{\log\frac{e(T + \lambda)}{\lambda}}{\log(1 + 1/\lambda)\lambda}$ is a strictly decreasing function of $\lambda$, and tends to $1$ as $\lambda\to\infty$. This suggests that as far as the $\epsilon$-related part is concerned, the larger the $\lambda$ the better.

The third part, involving the stochasticity $\sigma$, is 
\begin{align*}
    \mathcal{O}\left(\frac{\sigma \sqrt{T}}{(1 - \gamma)^{2.5}} \sqrt{\frac{d_{\lambda}\log\frac{e(T + \lambda)}{\lambda}\left(d_{\lambda}\log\frac{e(T + \lambda)}{\lambda p}{+}c_{\lambda}\right)}{\log(1 + 1/\lambda)\lambda}}\right).
\end{align*}
Here both $d_{\lambda}$ and $c_{\lambda}$ upper bound a non-increasing function of $\lambda$. $\frac{\log\frac{e(T + \lambda)}{\lambda}}{\log(1 + 1/\lambda)\lambda}$, as discussed before, is a strictly decreasing function of $\lambda$ that tends to $1$ as $\lambda\to\infty$. $\log\frac{e(T + \lambda)}{\lambda p}$ is a strictly decreasing function of $\lambda$ that tends to $\log\frac{e}{p}$ as $\lambda\to\infty$. This suggests that as far as the $\sigma$-related part is concerned, the larger the $\lambda$ the better.

To conclude, both the $\epsilon$- and $\sigma$- related parts of the regret bound prefers larger $\lambda$; however, the $\rho$-related part in general calls for a $\lambda$ that is neither too large nor too small.

\begin{table*}[tb]
    \caption{Evaluating different versions of Q-learning on classic control environments. Each model is trained for 1000 steps and evaluated over 100 episodes after training. The numbers are presented in the format $\textsc{Mean}\pm\textsc{Std}$. Larger is better.}
\vskip 0.15in
\begin{center}
\begin{small}
\begin{sc}
\begin{tabular}{lllll}
\toprule
                              & MountainCar              & Pendulum                 & Acrobot                  & CartPole                \\
\midrule
\textnormal{DQN}              & $-200.00\pm 0.00$        & $-274.67 \pm 393.03$     & $-161.31 \pm 96.13$      & $149.68 \pm 4.75$      \\
\textnormal{Linear Kernel}         & $-200.00\pm 0.00$        & $-1413.57\pm 202.19$     & $-500.00\pm 0.00$                         & $9.32 \pm 0.70$                        \\
\textnormal{Gaussian RBF Kernel}  & \bm{$-132.55 \pm 17.02$} & \bm{$-167.01 \pm 97.32$} & \bm{$-107.23 \pm 35.20$} & \bm{$200.00 \pm 0.00$}  \\
\bottomrule
\end{tabular}
\end{sc}
\end{small}
\end{center}
\vskip -0.1in
\label{tab:exp}
\end{table*}

The following corollary is a direct consequence of Theorem~\ref{thm:main}, Lemma~\ref{lem:deff-linear}, Lemma~\ref{lem:covering-linear}.
\begin{corollary}[Regret for linear kernels]
\label{cor:linear}
In Theorem~\ref{thm:main}, if
$
    \mathcal{K}((s_1, a_1), (s_2, a_2)) = \phi(s_1, a_1)^{\intercal}\phi(s_2, a_2) 
    $
where $\phi: \mathcal{S}\times\mathcal{A}\to\mathbb{R}^d$ for some positive integer $d$ such that $\norm{\phi(s, a)}_2 \leq 1$ for any $(s, a)\in\mathcal{S}\times\mathcal{A}$. Then we can choose $d_{\lambda} = d$ and some  
$
    c_{\lambda} = \mathcal{O}\left(d^2\log\frac{eH(T + \lambda)}{\lambda}\right)
    $
such that~\eqref{eq:betrue} is always true, and therefore with probability at least $1 - p$,
\begin{align*}
\regret(T) &=  \mathcal{O}\vast( \sqrt{\frac{Td\log\frac{e(T + \lambda)}{\lambda}}{\log(1 + 1/\lambda)(1 - \gamma)^5}}~\cdot\\
&\left(\rho + \epsilon\sqrt{\frac{dT}{\lambda}} + \sigma d\sqrt{\frac{\log\frac{eH(T + \lambda)}{\lambda p}}{\lambda}}\right)\vast).
\end{align*}
\end{corollary}
\begin{remark}
\label{rem:linear}
Recall that in the general setting $\lambda$ should be neither too large nor too small. However, since here $d_{\lambda}$ is bounded by $d$, which is independent of $\lambda$, if $\lambda$ is sufficiently small such that $\lambda = \mathcal{O}(1/T)$, we can make the dependency on $\rho$ to be $ \frac{\rho\sqrt{Td}}{(1 - \gamma)^{2.5}}$. In other words, in the linear setting, if the MDP is deterministic and the dynamics can be exactly represented by functions in $\mathcal{H}$, then we can choose $\lambda = \mathcal{O}(1/T)$ such that
\begin{align*}
   \regret(T) = \mathcal{O}\left(\frac{\rho\sqrt{Td}}{(1 - \gamma)^{2.5}}\right). 
\end{align*}
Comparing this with choosing $\lambda = \Theta(1)$, a factor of $\sqrt{\log(eT)}$ is reduced.
\end{remark}

A perhaps more interesting example would be the regret bound for the widely-used Gaussian RBF kernel. The following corollary is a direct consequence of Theorem~\ref{thm:main}, Lemma~\ref{lem:deff-gaussian}, Lemma~\ref{lem:covering-gaussian-ihi}, Lemma~\ref{lem:covering-gaussian-ihsi}.

\begin{corollary}[Regret for Gaussian RBF kernels]
\label{cor:gaussian}
In Theorem~\ref{thm:main}, if 
\begin{align*}
    \mathcal{K}((s_1, a_1), (s_2, a_2)) = e^{-\eta\norm{\phi(s_1, a_1) - \phi(s_2, a_2)}_2^2}
    \end{align*}
for some $\eta \geq 0$ and $\phi:\mathcal{S}\times\mathcal{A}\to\mathbb{R}^n$ for some positive integer $n$ such that $\norm{\phi(s, a)}_2 \leq 1$ for any $(s, a)\in\mathcal{S}\times\mathcal{A}$. Then we can choose some  
\begin{align*}
    d_{\lambda} &= 
    \left(\mathcal{O}\left(\eta + \log\frac{e(T + \lambda)}{\lambda(1 - \gamma)}\right)\right)^n,
\\
    c_{\lambda} &= \mathcal{O}\left(d_{\lambda}^2\log\frac{e(T + \lambda)}{\lambda(1 - \gamma)}\right),
\end{align*}
such that~\eqref{eq:betrue} is always true, and therefore with probability at least $1 - p$,
\begin{align*}
\regret(T) &=  \mathcal{O}\vast( \sqrt{\frac{Td_{\lambda}\log\frac{e(T + \lambda)}{\lambda}}{\log(1 + 1/\lambda)(1 - \gamma)^5}}~\cdot\\
&\left(\rho + \epsilon\sqrt{\frac{d_{\lambda} T}{\lambda}} + \sigma d_{\lambda}\sqrt{\frac{\log\frac{eH(T + \lambda)}{\lambda p}}{\lambda}}\right)\vast).
\end{align*}
\end{corollary}
We see that the above regret bound only differs from the bound in Corollary~\ref{cor:linear} in that $d$ is replaced by $\mathcal{O}\left(\eta + \log\frac{e(T + \lambda)}{\lambda(1 - \gamma)}\right)$. This suggests that this new quantity indicates the effective dimensionality under the Gaussian RBF kernel at scale $\lambda$. Note that although a smaller $\eta$ would benefit the regret bound, an $\eta$ too small would violate the assumption on $\rho$ as described in Section~\ref{sec:setup}.
\begin{remark}
\label{rem:gaussian}
Again we can inspect how the $\rho$-related term in the bound,
\begin{align*}
    &\mathcal{O}\left( \frac{\rho\sqrt{T}}{(1 - \gamma)^{2.5}}\cdot\sqrt{\frac{d_{\lambda}\log\frac{e(T + \lambda)}{\lambda}}{\log(1 + 1/\lambda)}}\right)
\end{align*}
depends on $\lambda$. Consider choosing $\lambda = \Theta(1/T)$, we have
\begin{align*}
    \frac{d_{\lambda}\log\frac{e(T + \lambda)}{\lambda}}{\log(1 + 1/\lambda)} = \mathcal{O}(d_{\lambda})
    =\left(\mathcal{O}\left(\eta + \log\frac{eT}{1 - \gamma}\right)\right)^n.
\end{align*}
In other words, in the Gaussian RBF kernel setting, if the MDP is deterministic and the dynamics can be exactly represented by functions in $\mathcal{H}$, then we can choose $\lambda = \Theta\left(1/T\right)$ such that
\begin{align*}
&\regret(T) = \frac{\rho\sqrt{T}}{(1 - \gamma)^{2.5}}\cdot
\left(\mathcal{O}\left(\eta + \log\frac{eT}{1 - \gamma}\right)\right)^{n/2}. 
\end{align*}
Comparing this with choosing $\lambda = \Theta(1)$, a factor of $\sqrt{\log(eT)}$ is reduced. Recall that in Remark~\ref{rem:linear}, we showed that the same thing happens in the linear setting. 
\end{remark}

\section{Experiments}
\label{sec:exp}
We test KQL on a suite of classic control tasks included in OpenAI Gym~\cite{brockman2016openai}: \textsc{MountainCar}, \textsc{Pendulum}, \textsc{Acrobot}, and \textsc{CartPole}. The action space of \textsc{Pendulum} is discretized to $\{-1, 0, 1\}$, all other environments have discrete action space natively.

\subsection{Methodology}
For any state $s$ and action $a$, let $[s, a]$ be the concatenation of the state vector and the one-hot embedding of the action; let $l$ be the length of state vectors. The states are first normalized such that $\norm{s}_{\infty} \leq 1$ for all $s$.
We experiment with two types of kernels:
\begin{itemize}
    \item Linear Kernel 
    \begin{align*}
        \mathcal{K}((s_1, a_1), (s_2, a_2)) = \frac{[s_1, a_1]^{\intercal}[s_2, a_2]}{2(l + 1)} + \frac{1}{2}.
    \end{align*}
    \item Gaussian RBF Kernel
    \begin{align*}
        \mathcal{K}((s_1, a_1), (s_2, a_2)) = e^{-\eta \norm*{[s_1, a_1] - [s_2, a_2]}_2^2}.
    \end{align*}
    where $\eta$ is empirically set to $0.02$ for \textsc{MountainCar}, \textsc{CartPole} and \textsc{Acrobot}, and $1$ for \textsc{Pendulum}.
\end{itemize}

We normalize the rewards to fall within $[0, 1]$. Since all environments are deterministic we can let $\sigma = 0$. We also believe that for these simple environments, if a correct kernel is used, $\epsilon$ can be made infinitesimally small. Our theoretical results (Theorem~\ref{thm:main}) suggest that in this case we can set 
\begin{align*}
    \beta = \frac{3\rho\sqrt{\lambda}}{1 - \gamma}.
\end{align*}
and our discussions in Remark~\ref{rem:linear} and Remark~\ref{rem:gaussian} suggest we can set
\begin{align*}
    \lambda = \mathcal{O}\left(\frac{1}{T}\right).
\end{align*}
 We heuristically set $\beta = \frac{\sqrt{\lambda}}{1 - \gamma}$ so that $\qt_t \gtrsim 1/(1 - \gamma)$ at the beginning of learning, and set $\lambda = \frac{1}{10T}$ so that $\lambda$ is large enough to not cause numerical issues.

We compare KQL with Deep Q-Learning (DQN)~\cite{mnih2013playing}, a practical, widely-used, neural network based algorithm known for its superior sample efficiency. We use the default implementation provided in Stable-Baselines3~\cite{raffin2021stable}, with the environment-specific parameter overrides from RL Baselines3 Zoo~\cite{rl-zoo3}, except the following to accommodate to the extreme small amount of samples: batch\_size=64, learning\_starts=100, target\_update\_interval=10, train\_freq=1, gradient\_steps=32.

To emphasize the test on sample efficiency, for each environment, all algorithms are only allowed 1000 \emph{steps} through the OpenAI Gym interface; a \emph{reset} is only allowed in the very beginning and after a \emph{done} is received. We use the same discounting factor $\gamma = 0.95$ for all algorithms and all environments. 

\subsection{Results}
The results are shown in Table~\ref{tab:exp}. As we can see, the linear kernel performs poorly on all tasks. In fact, its performance is not noticeably better than a random policy. Without handcrafted or learned feature embedding, the value function is highly non-linear; thus linear kernel suffers from huge approximation error. DQN makes progress on most tasks, with the exception of \textsc{MountainCar}, which requires more than random exploration in order to succeed with a very limited amount of interaction. Remarkably, the Gaussian RBF kernel achieves high return and performs best on all tasks, compensating for its long running time by its far superior sample efficiency.

\section{Where Next}
We have demonstrated that KQL is both theoretically sound and empirically promising. However, kernel methods are known to suffer from \emph{the curse of kernelization} --- their time complexity has a super-linear dependency on the number of samples, KQL is no exception. 

There are two standard ways to trade sample efficiency for time complexity in kernel-based online learning: one is sparsification, which aims at keeping the number of support vectors $d\ll T$ by various strategies~(see e.g.~\citet{engel2004kernel,sun2012size,calandriello2017second} and the discussion in~\citet[Related Work]{lu2016large}); the other is kernel approximation, which aims at projecting vectors in an RKHS into a Euclidean space, for example, using random Fourier features, and reduce the kernel setting to the linear setting~\cite{lu2016large,jezequel2019efficient}. Given its already impressive performance even without any kernel sparsification or approximation, it is not hard to believe that with proper acceleration, KQL is able to tackle more challenging tasks such as Atari games~\cite{bellemare2013arcade}. We leave it as future work to investigate these directions.

\bibliography{ref}
]

\appendix

\section{Proof of Lemma~\ref{lem:covering-gaussian-ihsi}}
Let $\mathbb{N}$ be the set of non-negative integers. For any multi-index $\nu = (n_1, n_2, \cdots, n_d)\in\mathbb{N}^d$, define 
$
\abs{\nu} = \sum_{i = 1}^n n_i
$.
The following facts will be useful.
\begin{fact}[Multinomial expansion]
\label{fact:multi}
For any nonnegative integer $n$,
\begin{align*}
\sum_{\substack{\nu\in\mathbb{N}^d\\\abs{\nu}= n}}\prod_{j = 1}^d \frac{(2\eta x_j^2)^{n_j}}{n_j !} = \frac{1}{n!}\left(2\eta\norm{x}_2^2\right)^n.
\end{align*}
\end{fact}
\begin{fact}[Taylor expansion]
\label{fact:taylor}
For any nonnegative integer $n$ and $t\geq 0$,
\begin{align*}
\sum_{n = N}^{\infty}\frac{t^n}{n!}\leq \frac{t^N}{N!}e^t.
\end{align*}
\end{fact}
Define the function 
\begin{align*}
e_{\nu}: X&\to\mathbb{R}\\
x&\mapsto  \prod_{j = 1}^d e_{n_j}(x_j)
\end{align*}
where
\begin{align*}
   e_{n_j}(x_j) = \sqrt{\frac{(2\eta)^{n_j}}{n_j!}} x_j^{n_j} e^{-\eta x_j^2}.
   \end{align*}
It is well-known that $B = \left\{e_{\nu}: \nu\in\mathbb{N}^d\right\}$ is an orthonormal basis for $\mathcal{H}$ (see e.g. \citet[Theorem 3.7]{steinwart2006explicit}). Consequently, $\left\{e_{\nu_1}\otimes e_{\nu_2}: \nu_1, \nu_2\in\mathbb{N}^d\right\}$ is an orthonormal (Schauder) basis for $\bhsh$. Now note that
\begin{align}
\begin{aligned}
    \norm*{\ihsi}^2 &= \sup_{\substack{T\in \bhsh\\\hsnorm{T}\leq 1}}\sup_{x\in X} \hinner{T\mathcal{K}_x, \mathcal{K}_x}^2\\
    &= \sup_{x\in X}\sup_{\substack{T\in \bhsh\\\hsnorm{T}\leq 1}} \left(\sum_{\nu_1, \nu_2\in \mathbb{N}^d}\hsinner*{T, e_{\nu_1}\otimes e_{\nu_2}}\cdot e_{\nu_1}(x)\cdot e_{\nu_2}(x)\right)^2\\
    &\stackrel{(a)}{=} \sup_{x\in X}\sum_{\nu_1, \nu_2\in \mathbb{N}^d} e_{\nu_1}^2(x)\cdot e_{\nu_2}^2(x)\\
    &= \sup_{x\in X}\left(\sum_{\nu\in B} e_{\nu}^2(x)\right)^2\\
    &=\sup_{x\in X} \left(\sum_{\nu\in \mathbb{N}^d}\prod_{j = 1}^d\frac{(2\eta x_j^2)^{n_j}}{n_j!}\cdot e^{-2\eta\norm{x}_2^2}\right)^2\\
    &\stackrel{(b)}{=}\sup_{x\in X} \left(\sum_{n = 0}^{\infty}\frac{1}{n!}\left(2\eta\norm{x}_2^2\right)^n\cdot e^{-2\eta\norm{x}_2^2}\right)^2\\
    &= 1,
\end{aligned}\label{eq:base-arg}
\end{align}
where (a) is due to Cauchy-Schwarz and (b) is due to Fact~\ref{fact:multi}.

Next we introduce for all positive integer $N$ two orthogonal projections: $P_N$, which is the projection onto span $\left\{e_{\nu_1}\otimes e_{\nu_2}: \abs{\nu_1} < N, \abs{\nu_2} < N\right\}$, and $Q_N$, which is the projection onto $P_N^{\bot}$. Note that
\begin{align*}
    \norm*{\ihsi Q_N}^2
    &=\sup_{\substack{T\in \bhsh\\\hsnorm{T}\leq 1}}\sup_{x\in X} \hinner{Q_N T\mathcal{K}_x, \mathcal{K}_x}^2\\
    &= \sup_{x\in X}\sup_{\substack{T\in \bhsh\\\hsnorm{T}\leq 1}} \left(\sum_{\substack{\nu_1, \nu_2\in \mathbb{N}^d\\\textnormal{$\norm{\nu_1}\geq N$ or $\norm{\nu_2}\geq N$}}}\hsinner*{Q_N T, e_{\nu_1}\otimes e_{\nu_2}}\cdot e_{\nu_1}(x)\cdot e_{\nu_2}(x)\right)^2\\
    &\stackrel{(a)}{=} \sup_{x\in X} \sum_{\substack{\nu_1, \nu_2\in \mathbb{N}^d\\\textnormal{$\norm{\nu_1}\geq N$ or $\norm{\nu_2}\geq N$}}}e^2_{\nu_1}(x)\cdot e^2_{\nu_2}(x)\\
    &\stackrel{(b)}{=} \sup_{x\in X} \vast(\sum_{n = N}^{\infty}\frac{1}{n!}\left(2\eta\norm{x}_2^2\right)^n\cdot e^{-2\eta\norm{x}_2^2}\cdot \sum_{n = 0}^{N - 1}\frac{1}{n!}\left(2\eta\norm{x}_2^2\right)^n\cdot e^{-2\eta\norm{x}_2^2}\\
    & \ \ \ \ \ \ \ \ + \sum_{n = 0}^{N - 1}\frac{1}{n!}\left(2\eta\norm{x}_2^2\right)^n\cdot e^{-2\eta\norm{x}_2^2}\cdot \sum_{n = N}^{\infty}\frac{1}{n!}\left(2\eta\norm{x}_2^2\right)^n\cdot e^{-2\eta\norm{x}_2^2}\\
    & \ \ \ \ \ \ \ \ + \sum_{n = N}^{\infty}\frac{1}{n!}\left(2\eta\norm{x}_2^2\right)^n\cdot e^{-2\eta\norm{x}_2^2}\cdot \sum_{n = N}^{\infty}\frac{1}{n!}\left(2\eta\norm{x}_2^2\right)^n\cdot e^{-2\eta\norm{x}_2^2}
    \vast)\\
    &\leq 2\sup_{x\in X} \left(\sum_{n = N}^{\infty}\frac{1}{n!}\left(2\eta\norm{x}_2^2\right)^n\cdot e^{-2\eta\norm{x}_2^2}\right)\\
    &\stackrel{(c)}{\leq} 2\sup_{x\in X}\left(\frac{(2\eta\norm{x}_2^2)^N}{N!}\right)\\
    &\leq 2\left(\frac{2e\eta}{N}\right)^N,
\end{align*}
where (a) is due to Cauchy-Schwarz and (b) is due to Fact~\ref{fact:multi} and (c) is due to Fact~\ref{fact:taylor}. Note that we can choose 
\begin{align*}
    N = \left\lceil2\left(\ln\frac{2\sqrt{2}}{\epsilon}+ e^2\eta\right)\right\rceil
    \end{align*}
such that $\norm*{\ihsi Q_N} \leq \epsilon / 2$. Finally, note that
\begin{align*}
    \ln \mathcal{N}(\epsilon, \ihsi) &\leq \ln\left(\mathcal{N}(\epsilon / 2, \ihsi P_N) \cdot \mathcal{N}(\epsilon / 2, \ihsi Q_N)\right)\\
    &=\ln \mathcal{N}(\epsilon / 2, \ihsi P_N)\\
    &\stackrel{(a)}{\leq} \textnormal{rank}(P_N)\ln\left(1 + \frac{4}{\epsilon}\right)\\
    &\leq N^{2d}\ln\left(1 + \frac{4}{\epsilon}\right)\\
    &= \left\lceil2\left(\ln\frac{2\sqrt{2}}{\epsilon}+ e^2\eta\right)\right\rceil^{2d}\ln\left(1 + \frac{4}{\epsilon}\right),
\end{align*}
where in (a) we used $\norm{\ihsi P_N}\leq \norm{\ihsi} = 1$ and the fact that a $n$-dimensional unit ball can be $\epsilon$-covered by $\left(1 + 2/\epsilon\right)^n$ points.
\section{Proof of Other Lemmas}

\begin{proof}[Proof of Lemma~\ref{lem:deff-identity}]
Let $W = \left(\Sigma + \lambda I\right)^{-1}$, note that
\begin{align*}
\sum_{i = 1}^{n} \norm*{\kxi}^2_{W} &= \sum_{i = 1}^{n}\hinner{W\kxi, \kxi}\\
&=\sum_{i = 1}^{n} \tr \left(W \left(\kxi\otimes\kxi\right)\right)\\
&=\tr\left(W \sum_{i = 1}^{n}\left(\kxi\otimes\kxi\right)\right)\\
&\stackrel{(a)}{=}\deff(\lambda, x_{1:n}),
\end{align*}
where (a) is due to Lemma~\ref{lem:deff-equivalent}.
\end{proof}

\begin{proof}[Proof of Lemma~\ref{lem:deff-non-decreasing}]
Let $W_n = \left(\sum_{i = 1}^n \kxi\otimes\kxi + \lambda I\right)^{-1}$ and $f = W_{n - 1}\kxn$. It is easy to verify that
\begin{align}
    W_n = W_{n - 1} - \frac{f\otimes f}{1 + \norm{\kxn}^2_{W_{n - 1}}}.\label{eq:inc-inverse}
\end{align}
Therefore,
\begin{align*}
    \deff(\lambda, x_{1:n}) - \deff(\lambda, x_{1:(n - 1)})
    & \stackrel{(a)}{=} \sum_{i = 1}^n \norm{\kxi}^2_{W_n} - \sum_{i = 1}^{n - 1} \norm{\kxi}^2_{W_{n - 1}}\\
    &\stackrel{(b)}{=}\norm{\kxn}^2_{W_{n - 1}} - \sum_{i = 1}^{n}\frac{\hinner*{\kxi, f}^2}{1 + \norm{\kxn}^2_{W_{n - 1}}}\\
    &=\frac{\norm{\kxn}^2_{W_{n - 1}}}{1 + \norm*{\kxn}^2_{W_{n - 1}}} - \sum_{i = 1}^{n - 1}\frac{\hinner*{\kxi, f}^2}{1 + \norm{\kxn}^2_{W_{n - 1}}}\\
    &=\frac{\norm*{\kxn}^2_{W_{n - 1}}}{1 + \norm*{\kxn}^2_{W_{n - 1}}} - \frac{\hinner*{\left(\sum_{i = 1}^{n - 1}\kxi\otimes\kxi\right)f, f}}{1 + \norm{\kxn}^2_{W_{n - 1}}}\\
    &\geq\frac{\norm*{\kxn}^2_{W_{n - 1}}}{1 + \norm*{\kxn}^2_{W_{n - 1}}} - \frac{\norm*{\kxn}^2_{W_{n - 1}}}{1 + \norm{\kxn}^2_{W_{n - 1}}}\\
    &=0,
\end{align*}
where (a) is due to Lemma~\ref{lem:deff-identity} and (b) is due to~\eqref{eq:inc-inverse}.
\end{proof}

\begin{proof}[Proof of Lemma~\ref{lem:deff-linear}]
Let $\Sigma = \sum_{i = 1}^n x_i x_i^{\intercal}$. By Lemma~\ref{lem:deff-equivalent} we have that $\deff(\lambda, x_{1:n}) = \tr\left((\Sigma + \lambda I)^{-1}\Sigma\right)$. Let $\lambda_1, \lambda_2, \cdots, \lambda_d$ be the eigenvalues of $\Sigma$, we have $\tr\left((\Sigma + \lambda I)^{-1}\Sigma\right) = \sum_{i = 1}^{n}\frac{\lambda_i}{\lambda_i + \lambda} \leq d$.
\end{proof}

\begin{proof}[Proof of Lemma~\ref{lem:sum-of-ucb-bound}]
First note that since $\Sigma_{i - 1} / \lambda$ is a linear operator of rank at most $i - 1$, $\det\left(\Sigma_{i - 1} / \lambda + I\right)$ is well-defined. 
Therefore,
\begin{align}
\begin{aligned}
\norm*{\kxi}^2_{\left(\Sigma_{i - 1} + \lambda I\right)^{-1}} &= \norm*{\kxi / \lambda}^2_{\left(\Sigma_{i - 1}/\lambda + I\right)^{-1}}\\
&\stackrel{(a)}{=}  \frac{e^{\dpse(\lambda, x_{1:i})}}{ e^{\dpse(\lambda, x_{1:(i - 1)})}} - 1,
\end{aligned}\label{eq:fracdet}
\end{align}
where (a) can be proved for example using the same argument used for \citet[Lemma 11.11]{cesa2006prediction}. Note that~\eqref{eq:fracdet} already implies (I), since a norm is non-negative. Next, 
it is easy to verify that $\norm*{\kxi}^2_{\left(\Sigma_{i - 1} + \lambda I\right)^{-1}} \leq 1 / \lambda$ and that $x \leq \frac{b}{\ln(1 + b)}\ln(1 + x)$ for any $x\in[0, b]$. Therefore,  
\begin{align*}
\norm*{\kxi}^2_{\left(\Sigma_{i - 1} + \lambda I\right)^{-1}} 
&\leq \frac{1/\lambda}{\ln\left(1 + 1/\lambda\right)}\cdot\ln\left(1 + \norm*{\kxi}^2_{\left(\Sigma_{i - 1} + \lambda I\right)^{-1}}\right)\\
&\stackrel{(a)}{=} \frac{1/\lambda}{\ln\left(1 + 1/\lambda\right)}\cdot\left(\dpse(\lambda, x_{1:i}) - \dpse(\lambda, x_{1:(i-1)})\right),
\end{align*}
where (a) is due to~\eqref{eq:fracdet}. Therefore (I) is proved, since a norm cannot be negative. Note that the above argument is still true if we replace $x_i$ by an arbitrary $x\in X$. Consequently, denote by $[a, b]$ the concatenation of sequence $a$ and sequence $b$, we have
\begin{align*}
    &\sum_{i = 1}^n\sum_{\tau = 1}^{i}\gamma^{i - \tau}\norm*{\kxi}^2_{\left(\Sigma_{\tau-1} + \lambda I\right)^{-1}}\\ 
    &=\frac{2/\lambda}{\ln\left(1 + 1/\lambda\right)}\left(\sum_{i = 1}^n\sum_{\tau = 1}^i\gamma^{i - \tau}\dpse(\lambda, [x_{1:(\tau - 1)}, x_i]) - \sum_{i = 1}^n\sum_{\tau = 1}^i \gamma^{i - \tau}\dpse(\lambda, x_{1:(\tau - 1)})\right)\\
    &\stackrel{(a)}{\leq}\frac{2/\lambda}{\ln\left(1 + 1/\lambda\right)}\left(\underbrace{\sum_{i = 1}^n\sum_{\tau = 1}^i\gamma^{i - \tau}\dpse(\lambda, x_{1:i})}_{(A)} - \underbrace{\sum_{i = 1}^n\sum_{\tau = 1}^i \gamma^{i - \tau}\dpse(\lambda, x_{1:(\tau - 1)})}_{(B)}\right)\\
    &\stackrel{(b)}{\leq} \frac{2/\lambda}{\ln\left(1 + 1/\lambda\right)}\left(\sum_{i = 1}^{n}\sum_{\tau = 1}^{2i -n} \gamma^{i - \tau}\dpse(\lambda, x_{1:i})\right)\\
    &\stackrel{(c)}{\leq} \frac{2/\lambda}{\ln\left(1 + 1/\lambda\right)}\left(\sum_{i = 1}^{n} \gamma^{n - i}\cdot\frac{\dpse(\lambda, x_{1:n})}{1 - \gamma}\right)\\
    &\leq\frac{2/\lambda}{\ln\left(1 + 1/\lambda\right)(1 - \gamma)^2}\cdot\dpse(\lambda, x_{1:n}),
\end{align*}
where we used (I) in (a), (b), (c) and the nonnegativity of $\dpse$ in (b); in particular, in (b) we canceled the $(i, \tau)$th term in (A) with the $(2i + 1 - \tau, i + 1)$th term in (B), where the first index of a term is its position in the first summation and the second index is its position in the second summation. This concludes the proof of (II).
\end{proof}

\begin{proof}[Proof of Lemma~\ref{lem:covering-linear}]
In the case of linear kernel, for any $x\in\mathcal{H}$, $\hnorm{x} = \norm{x}_{*} = \norm{x}_2$; and for any linear operator $T$ over $X$, $\norm{T}_*\leq \norm{T}_{\textnormal{F}} = \hsnorm{T}$. Because a $n$-dimensional unit ball can be $\epsilon$-covered by $(1 + 2/\epsilon)^n$ points, the lemma follows immediately.
\end{proof}

\section{Proof of Main Results}
We will prove Theorem~\ref{thm:main} in this section. We first introduce some notations. For any $t\in[T]$, let $\prev(t) = [t - 1]$.
Let
\begin{align*}
    \qt_t(s, a) &= \clip_{[0, 1/(1 - \gamma)]}\left(\qh_t(s, a) + \beta\norm*{\mathcal{K}_{(s, a)}}_{\wh_t}\right),\\
    \vt_t(s) &= \max_a \qt_t(s, a).
\end{align*}
Define
\begin{align*}
    \Phi &=
    \frac{1}{1 - \gamma}\cdot\left(3\rho\sqrt{\lambda} + \epsilon\sqrt{Td_{\lambda}}\right)\\
    \Psi_{p} &= \frac{2\sigma}{1 - \gamma}\cdot\sqrt{ 
    d_{\lambda}\ln\frac{e(T + \lambda)}{\lambda} +  \ln\frac{1}{p}+c_{\lambda}}.
\end{align*}
For any $p > 0$, let $\mathfrak{E}_p$ be the event that for any $s\in\mathcal{S}$, $a\in\mathcal{A}$, $t\in[T]$,
\begin{align*}
    &\abs*{\qh_t(s, a) - Q_*(s, a) - \gamma\left(\left(M_*\left(\vt_{t-1} - V_*\right)\right)(s, a)\right)}\\
    &\leq \left(\Phi + \Psi_{p}\right)\cdot \norm*{\mathcal{K}_{(s, a)}}_{\wh_t} + \frac{2\epsilon}{1 - \gamma}.
\end{align*} 
where $\vt_0$ can be any function that is bounded in $[0, 1/(1 - \gamma)]$.
We start from some auxiliary lemmas.
\begin{lemma}
\label{lem:w-norm-bound}
For any $t\in[T]$ and $f\in\mathcal{H}$, $\frac{\hnorm{f}}{\sqrt{\lambda + T}} \leq \norm*{f}_{\wh_t} \leq \frac{\hnorm*{f}}{\sqrt{\lambda}}$.
\end{lemma}
\begin{proof}
Let $\Sigma = \sum_{\tau \in\prev(t)}\kth\otimes\kth$ and $\lambda_1, \lambda_2, \cdots, \lambda_{\abs{\prev(t)}} \geq 0$ be the eigenvalues of $\Sigma$. Note that the eigenvalues of $\wh_t = \left(\Sigma + \lambda I\right)^{-1}$ are $\frac{1}{\lambda_1 + \lambda}, \frac{1}{\lambda_2 + \lambda}, \cdots, \frac{1}{\lambda_{\abs{\prev(t)}} + \lambda}$.  The upper bound follows immediately. To get the upper bound, note that for each $\lambda_i$, there exists $f\in\mathcal{H}$, $\hnorm{f} = 1$, such that
\begin{align*}
    \lambda_i &= \hnorm{\Sigma f}\\
              &\stackrel{(a)}{\leq} \sum_{\tau\in\prev(t)}\hnorm{f}\hnorm{\kth}^2\\
              &\stackrel{(b)}{\leq} T,
\end{align*}
where (a) is due to Cauchy-Schwarz and (b) is because by our assumption $\hnorm{\kth} \leq 1$. The lower bound then follows.
\end{proof}

Because by our assumptions $\mathcal{K}$ is normalized such that for any $(s, a)\in\mathcal{S}\times\mathcal{A}$, $\hnorm*{\mathcal{K}_{(s, a)}} \geq 1/2$, we have the following corollary of Lemma~\ref{lem:w-norm-bound}.
\begin{corollary}
\label{cor:bound-for-beta}
For any $t\in[T]$, $(s, a)\in\mathcal{S}\times\mathcal{A}$, $\norm*{\mathcal{K}_{(s, a)}}_{\wh_t} \geq \frac{1}{2\sqrt{T + \lambda}}$.
\end{corollary}
\begin{lemma}
\label{lem:hsnorm-bound}
For any $t\in[T]$, $\hsnorm*{\wh_t} \leq \frac{\sqrt{T}}{\lambda}$.
\end{lemma}
\begin{proof}
Let $\Sigma = \sum_{\tau\in\prev(t)}\kth\otimes\kth$ and $\lambda_1, \lambda_2, \cdots, \lambda_{\abs{\prev(t)}} \geq 0$ be the eigenvalues of $\Sigma$. Note that the eigenvalues of $\wh_t = \left(\Sigma + \lambda I\right)^{-1}$ are $\frac{1}{\lambda_1 + \lambda}, \frac{1}{\lambda_2 + \lambda}, \cdots, \frac{1}{\lambda_{\abs{\prev(t)}} + \lambda}$, and
\begin{align*}
    \hsnorm*{\wh_t} &= \sqrt{\sum_{\tau\in\prev(t)}\frac{1}{(\lambda_i + \lambda)^2}}\\
    &\leq\frac{\sqrt{T}}{\lambda}.
\end{align*}
\end{proof}

\begin{lemma}
\label{lem:qhnorm-bound}
    For any $t\in[T]$, $\hnorm*{\qh_t} \leq \frac{T}{\lambda(1 - \gamma)}$.
\end{lemma}
\begin{proof}
For any $f\in\mathcal{H}$, we have that
\begin{align*}
    \abs*{\qh_t f} &= \abs*{\left(\sum_{\tau\in\prev(t)}\left(r_h(s_{\tau}, r_{\tau}) + \vt_{t-1}(s_{\tau + 1})\right)\cdot\mathcal{K}_{(s_{\tau}, a_{\tau})}\right)\wh_t f}\\
    &\leq \frac{1}{1 - \gamma}\cdot\sum_{\tau\in\prev(t)}\abs*{\kth\wh_t f}\\
    &\stackrel{(a)}{\leq} \frac{1}{1 - \gamma}\cdot\sum_{\tau\in\prev(t)}\norm*{f}_{\wh_t}\cdot \norm*{\kth}_{\wh_t}\\
    &\stackrel{(b)}{\leq} \frac{1}{1 - \gamma}\cdot\sqrt{\sum_{\tau\in\prev(t)}\norm*{f}^2_{\wh_t}}\cdot\sqrt{\sum_{\tau\in\prev(t)}\norm*{\kth}^2_{\wh_t}}\\
    & \stackrel{(c)}{\leq} \frac{T}{\lambda(1 - \gamma)}\cdot\hnorm{f},
\end{align*}
where (a) and (b) are due to Cauchy-Schwarz, and (c) is due to the upper bound in Lemma~\ref{lem:w-norm-bound}.
\end{proof}

\begin{lemma}
\label{lem:est-lower-bound}
For any $p > 0$, if 
\begin{align*}
\beta= 
\min\left(\Phi + \Psi_{p}, \frac{2\sqrt{T + \lambda}}{1 - \gamma}\right)
\end{align*}
and $\mathfrak{E}_{p}$ happens, we have for any $t\in[T]$, $s\in\mathcal{S}$, $a\in\mathcal{A}$,
\begin{align*}
    Q_*(s, a) - \frac{2\epsilon}{1 - \gamma}\cdot\sum_{\tau = 1}^t \gamma^{t - \tau}\leq \qt_t(s, a) \leq Q_*(s, a) + 2\beta\sum_{\tau = 1}^t\gamma^{t - \tau}\norm{\mathcal{K}_{(s, a)}}_{\wh_{\tau}} + \frac{2\epsilon}{1 - \gamma}\cdot\sum_{\tau = 1}^t \gamma^{t - \tau}.
\end{align*}
\end{lemma}
\begin{proof}
    We prove the lemma by induction. When $t = 1$, since $\vt_0$ can be chosen as either $0$ or $1/(1 - \gamma)$, by the definition of $\mathfrak{E}_p$ and $\qt_t$, the choice of $\beta$, and Corollary~\ref{cor:bound-for-beta}, the inequalities hold. Now the inequality holds for $t = t' - 1$ where $1 < t' \leq T$, we have again by the definition of $\mathfrak{E}_p$, the choice of $\beta$, and Corollary~\ref{cor:bound-for-beta},
    \begin{align}
    \begin{aligned}
    &\qh_{t'}(s, a) + (\Phi + \Psi_{p})\norm*{\mathcal{K}_{(s, a)}}_{\wh_{t'}}\\ 
    &\in Q_*(s, a) +\gamma\left(M_*\left(\vt_{t'-1} - V_*\right)\right)(s, a) + \left[- \frac{2\epsilon}{1 - \gamma},~2\beta\norm*{\mathcal{K}_{(s, a)}}_{\wh_{t'}} + \frac{2\epsilon}{1 - \gamma}\right].
    \end{aligned}
    \label{eq:ind-go-back}
    \end{align}
Note that on one hand
\begin{align*}
    \left(M_*\left(\vt_{t'-1} - V_*\right)\right)(s, a) &= \mathbb{E}_{s'\sim\mathbb{P}_(s, a)}\left[\vt_{t'-1}(s') - V_*(s')\right]\\
    &\geq \mathbb{E}_{s'\sim\mathbb{P}(s, a)}\left[\qt_{t'-1}(s', a^*(s')) - Q_*(s', a^*(s'))\right]\\
    &\stackrel{(a)}{\geq}-\frac{2\epsilon}{1 - \gamma}\cdot\sum_{\tau = 1}^{t' - 1}\gamma^{t' - 1 - \tau},
\end{align*}
where (a) is by induction hypothesis. On the other hand, for any $s'\in\mathcal{S}$, let $\tilde{a}(s')$ be an arbitray element in $\argmax \qt_{t' - 1}(s', a)$, we have
\begin{align*}
    \left(M_*\left(\vt_{t'-1} - V_*\right)\right)(s, a) &= \mathbb{E}_{s'\sim\mathbb{P}_(s, a)}\left[\vt_{t'-1}(s') - V_*(s')\right]\\
    &\leq \mathbb{E}_{s'\sim\mathbb{P}(s, a)}\left[\qt_{t'-1}(s', \tilde{a}(s')) - Q_*(s', \tilde{a}(s'))\right]\\
    &\stackrel{(a)}{\leq}2\beta\sum_{\tau = 1}^{t' - 1}\gamma^{t' - 1 - \tau}\norm*{\mathcal{K}_{(s, a)}}_{\wh_{\tau}} + \frac{2\epsilon}{1 - \gamma}\cdot\sum_{\tau = 1}^{t' - 1}\gamma^{t' - 1 - \tau},
\end{align*}
Going back to~\eqref{eq:ind-go-back} we arrive at
\begin{align*}
    &\qh_{t'}(s, a) + (\Phi + \Psi_{p})\norm*{\mathcal{K}_{(s, a)}}_{\wh_{t'}}\\ &\in Q_{*}(s, a) +\left[ -\frac{2\epsilon}{1 - \gamma}\cdot\sum_{\tau = 1}^{t' - \tau}\gamma^{\tau},~2\beta\sum_{\tau = 1}^{t'}\gamma^{t' - \tau}\norm*{\mathcal{K}_{(s, a)}}_{\wh_{\tau}} + \frac{2\epsilon}{1 - \gamma}\cdot\sum_{\tau = 1}^{t'}\gamma^{t' - \tau}\right].
\end{align*}
Using the choice of $\beta$, Corollary~\ref{cor:bound-for-beta}, and the definition of $\qt_t$, we see that the inequality holds for $t = t'$. This concludes the proof.
\end{proof}

\begin{lemma}
\label{lem:covering}
Let $\mathcal{Q}$ be the function class containing all functions
\begin{align*}
    \mathcal{S}\times\mathcal{A}&\to\mathbb{R}\\
   (s, a)&\mapsto f(s, a) + \norm{\mathcal{K}_{(s, a)}}_{W},
\end{align*}
where $f\in\mathcal{H}$ is such that $\hnorm{f} \leq \frac{T}{\lambda(1 - \gamma)}$ and $W\in\bhsh$ is such that $\hsnorm{W}\leq \frac{2(T + \lambda)}{\lambda(1 - \gamma)}$. 
Then for any $p > 0$, with probability at least $1 - p$, for any $t\in[T]$, and $q\in \mathcal{Q}$, 
\begin{align*}
\norm*{\sum_{\tau\in\prev(t)}\left(v_{q}(s_{\tau+1}) - \left(M_* v_{q}\right)(s_{\tau}, a_{\tau})\right)\cdot\kth}_{\wh_t} \leq  \sqrt{\lambda} + \Psi_{p},
\end{align*}
where $v_{q}(s) = \max_a\clip_{[0, 1/(1-\gamma)]}q(s, a)$.
\end{lemma}
\begin{proof}
First let us assume we have a $\xi > 0$ and a function class $\tilde{\mathcal{Q}}: \mathcal{S}\times\mathcal{A}\to\mathbb{R}$ of finite cardinality such that for any $q\in \mathcal{Q}$ there exists a $\gamma(q)\in\tilde{\mathcal{Q}}$ such that $\norm{q - \gamma(q)}_{\infty} \leq \xi$. Then with probability at least $1 - p$, for any $t\in[T]$, and $q\in \mathcal{Q}$,
\begin{align*}
&\norm*{\sum_{\tau \in\prev(t)}\left(v_{q}(s_{\tau + 1}) - \left(M_* v_{q}\right)(s_{\tau, h}, a_{\tau, h})\right)\cdot\kth}_{\wh_t} \\
&\leq \norm*{\sum_{\tau\in\prev(t)}\left(v_{\gamma(q)}(s_{\tau + 1}) - \left(M_* v_{\gamma(q)}\right)(s_{\tau, h}, a_{\tau, h})\right)\cdot\kth}_{\wh_t} + \norm*{\sum_{\tau\in\prev(t)}2\xi\cdot\kth}_{\wh_t}\\
&\stackrel{(a)}{\leq} \norm*{\sum_{\tau\in\prev(t)}\left(v_{\gamma(q)}(s_{\tau + 1}) - \left(M_* v_{\gamma(q)}\right)(s_{\tau}, a_{\tau})\right)\cdot\kth}_{\wh_t} + \frac{1}{\sqrt{\lambda}}\hnorm*{\sum_{\tau\in\prev(t)}2\xi\cdot\kth}\\
&\stackrel{(b)}{\leq} \frac{\sqrt{2}\sigma}{1 - \gamma} \sqrt{d_{\lambda}\ln \frac{e(T + \lambda)}{\lambda} + \ln\abs[\big]{\vt} + \ln\frac{1}{p}} + \frac{2\xi T}{\sqrt{\lambda}},
\end{align*}
where in (a) we used the upper bound in Lemma~\ref{lem:w-norm-bound}, in (b) we used Lemma~\ref{lem:self-normalized}, Lemma~\ref{lem:effective-dimension-bound}, and a union bound. 

It remains to choose $\vt$ and $\xi$. Note that for any $q_1$ (induced by $f_1$, $W_1$) and $q_2$ (induced by $f_2$, $W_2$) in $\mathcal{Q}$,
\begin{align*}
   \norm{q_1 - q_2}_{\infty} \leq \norm{f_1 - f_2}_* + \sqrt{\norm{W_1 - W_2}_*},
\end{align*}
so it suffices to bound $\norm{f_1 - f_2}_*$ by $\xi / 2$ and bound $\norm{W_1 - W_2}_*$ by $\xi^2 / 4$, therefore we can choose $\vt$ such that
\begin{align*}
    \ln\abs[\big]{\vt} \leq \ln\mathcal{N}\left(\frac{\xi\lambda(1 - \gamma)}{2T}, \ihi\right) + \ln\mathcal{N}\left(\frac{\xi^2\lambda(1 - \gamma)}{8(T + \lambda)}, \ihsi\right),
\end{align*}
choosing $\xi = \frac{\lambda}{2T}$ concludes the proof.
\end{proof}

\begin{lemma}
\label{lem:est-dev}
For any $p > 0$, $\mathfrak{E}_{p}$ happens with probability at least $1 - p$.
\end{lemma}
\begin{proof}
First note that for any $t\in[T]$, $Q_* = r_* + \gamma M_* V_*$, and consequently
\begin{align*}
    \norm*{Q_* - \left(r + \gamma M V_*\right)}_{\infty} \leq \frac{\epsilon}{1 - \gamma}.
\end{align*}
To proceed, note that for any $t\in[T]$, 
\begin{align*}
    &\qh_t - \left(r + \gamma M V_*\right)\\ 
    & = \left(\sum_{\tau \in \prev(t)}\left(r(s_{\tau}, a_{\tau}) + \gamma\vt_{t-1}(s_{\tau + 1})\right)\cdot\mathcal{K}_{(s_{\tau}, a_{\tau})}  - \left(r + \gamma M V_*\right)\left(\wh_{t}\right)^{-1}\right)\wh_{t}\\
    &= \underbrace{-\lambda \left(r + \gamma M V_*\right) \wh_{t}}_{E_1}\\
    &+ \gamma \underbrace{\left(\sum_{\tau \in\prev(t)}\left(\vt_{t-1}(s_{\tau + 1}) - \left(M_* \vt_{t-1}\right)(s_{\tau}, a_{\tau})\right)\cdot\kth\right)\wh_{t}}_{E_2}\\
    &+ \gamma \underbrace{\left(\sum_{\tau \in\prev(t)} \left(M\left(\vt_{t-1} - V_*\right)\right)(s_{\tau}, a_{\tau})\cdot \kth\right)\wh_{t}}_{E_3} \\
    &+ \gamma \underbrace{\left(\sum_{\tau \in\prev(t)}\left(\left(M_* - M\right)\left(\vt_{t-1} - V_*\right)\right)(s_{\tau}, a_{\tau})\cdot\kth\right)\wh_t}_{E_4}.
\end{align*}
Let us bound $E_1(s, a)$, $E_2(s, a)$, $E_3(s, a)$, $E_4(s, a)$ separately. First note that
\begin{align*}
\abs*{E_1(s, a)} &\leq \lambda \abs*{r\wh_t\mathcal{K}_{(s, a)}} + \lambda\abs*{\left(\gamma M V_*\right)\wh_t\mathcal{K}_{(s, a)}}\\
&\stackrel{(a)}{\leq} \lambda \cdot\left(\norm*{r}_{\wh_t} + \norm*{\gamma M V_*}_{\wh_t}\right)\cdot\norm*{\mathcal{K}_{(s, a)}}_{\wh_t}\\
&\stackrel{(b)}{\leq} \sqrt{\lambda}\cdot \left(\hnorm*{r} + \hnorm*{\gamma M V_*}\right)\cdot\norm*{\mathcal{K}_{(s, a)}}_{\wh_t}\\
&\leq \sqrt{\lambda}\left(\hnorm{r_h} + \frac{\gamma}{1 - \gamma}\norm*{M}\right)\cdot\norm*{\mathcal{K}_{(s, a)}}_{\wh_t}\\
&\leq\frac{\sqrt{\lambda}\rho}{1 - \gamma} \cdot\norm*{\mathcal{K}_{(s, a)}}_{\wh_t},
\end{align*}
where (a) is due to Cauchy-Schwarz and (b) is due to the upper bound in Lemma~\ref{lem:w-norm-bound}. 

To bound $E_2$, recall that for any $t\geq 2$
\begin{align*}
    \vt_{t-1}(s, a) = \max_a\clip_{[0, 1/(1 - \gamma)]}\left(\qh_{t+1}(s, a) + \norm*{\mathcal{K}_{(s, a)}}_{\beta\wh_{t+1}}\right),
\end{align*}
and by Lemma~\ref{lem:qhnorm-bound} $\qh_{t+1}(s, a) \leq \frac{T}{\lambda(1 - \gamma)}$, and by Lemma~\ref{lem:hsnorm-bound} as well as the constraint on $\beta$, $\hsnorm{\beta\wh_{t + 1}}\leq \frac{2(T + \lambda)}{\lambda(1 - \gamma)}$. Therefore by Lemma~\ref{lem:covering} we have that with probability at least $1 - p$, uniformly for all $t\in[T]$,
\begin{align*}
    \abs*{E_2(s, a)} &=\abs*{\left(\sum_{\tau \in\prev(t)}\left(\vt_{t-1}(s_{\tau + 1}) - \left(M_* \vt_{t-1}\right)(s_{\tau}, a_{\tau})\right)\cdot\kth\right)\wh_t \mathcal{K}_{(s, a)}}\\
    &\leq \norm*{\sum_{\tau\in\prev(t)}\left(\vt_{t-1}(s_{\tau + 1}) - \left(M_* \vt_{t-1}\right)(s_{\tau}, a_{\tau})\right)\cdot\kth}_{\wh_t}\cdot\norm*{\mathcal{K}_{(s, a)}}_{\wh_t}\\
    &\leq \left( \sqrt{\lambda} + \Psi_{p} \right)\cdot\norm*{\mathcal{K}_{(s, a)}}_{\wh_t}. 
\end{align*}
Also note that
\begin{align*}
    E_3(s, a) &= \left(\sum_{\tau \in\prev(t)} \hinner*{M\left(\vt_{t-1} - V_{*}\right), \kth}\cdot\kth\right)\wh_t\mathcal{K}_{(s, a)}\\
    & = \left(\left(M\left(\vt_{t-1} - V_{*}\right)\right)\left(\sum_{\tau \in\prev(t)} \kth\otimes\kth\right) \right)\wh_t\mathcal{K}_{(s, a)}\\
    & = \left(M\left(\vt_{t-1} - V_*\right)\right)\mathcal{K}_{(s, a)} - \lambda\left(M\left(\vt_{t-1} - V_*\right)\right)\wh_t \mathcal{K}_{(s, a)}
\end{align*}
Therefore, we have that
\begin{align*}
\abs*{E_3(s, a) - \left(M\left(\vt_{t-1} - V_*\right)\right)(s, a)} &= \lambda\abs*{\left(M\left(\vt_{t-1} - V_*\right)\right)\wh_t \mathcal{K}_{(s, a)}}\\
&\stackrel{(a)}{\leq}\lambda\cdot\norm*{M\left(\vt_{t-1} - V_*\right)}_{\wh_t}\cdot\norm*{\mathcal{K}_{(s, a)}}_{\wh_t}\\
&\stackrel{(b)}{\leq}\sqrt{\lambda}\cdot\hnorm*{M\left(\vt_{t-1} - V_*\right)}\cdot\norm*{\mathcal{K}_{(s, a)}}_{\wh_t}\\
&\leq\frac{\sqrt{\lambda}}{1 - \gamma}\cdot\norm*{M}\cdot\norm*{\mathcal{K}_{(s, a)}}_{\wh_t}\\
&\leq\frac{\sqrt{\lambda}\rho}{1 - \gamma} \cdot\norm*{\mathcal{K}_{(s, a)}}_{\wh_t}
\end{align*}
where (a) is due to Cauchy-Scharz and (b) is due to the upper bound in Lemma~\ref{lem:w-norm-bound}. We also have
\begin{align*}
    \abs*{\left(\left(M - M_*\right)\left(\vt_{t-1} - V_*\right)\right)(s, a)} \leq \frac{\epsilon}{1 - \gamma} .
\end{align*}
Finally, note that
\begin{align*}
    \abs*{E_4(s, a)} &\leq \frac{\epsilon}{1 - \gamma}\sum_{\tau \in\prev(t)}\abs*{\kth\wh_t\mathcal{K}_{s, a}}\\
    &\stackrel{(a)}\leq \frac{\epsilon}{1 - \gamma}\sqrt{\sum_{\tau \in\prev(t)}\norm*{\kth}^2_{\wh_t}\cdot\sum_{\tau\in\prev(t)}\norm*{\mathcal{K}_{s, a}}^2_{\wh_t}}\\
    &\stackrel{(b)}\leq \frac{\epsilon\sqrt{Td_{\lambda}}}{1 - \gamma}\cdot\norm*{\mathcal{K}_{s, a}}_{\wh_t}
\end{align*}
where (a) is due to Cauchy-Schwarz and (b) is due to Lemma~\ref{lem:deff-identity} and Lemma~\ref{lem:deff-non-decreasing}.
Putting everything together concludes the proof.
\end{proof}
Now we are ready to prove our main theorem.
\begin{proof}[Proof of Theorem~\ref{thm:main}]
The proof will be conditioned on $\mathfrak{E}_{p/2}$, which happens with probability at least $p/2$ by Lemma~\ref{lem:est-dev}. Define
\begin{align*}
    \delta_{t} &= \vt_{t}(s_t) - V_t,\\
    \phi_{t} &= \vt_{t}(s_t) - V_*(s_t),\\
    \Delta_t &= V_*(s_t) - V_t = \delta_t - \phi_t,\\
    \zeta_{t} &= \left(M_*(V_*)\right)(s_t, a_t) - V_*(s_{t + 1}).
\end{align*}
Note that
\begin{align}
    \phi_t \geq \qt_t(s_t, a^*(s_t)) - Q_*(s_t, a^*(s_t)) \stackrel{(a)}{\geq} - \frac{2\epsilon}{(1 - \gamma)^2},\label{eq:phibound}
\end{align}
where (a) is due to Lemma~\ref{lem:est-lower-bound}.
We have for any $t\in[T]$,
\begin{align*}
    \delta_{t} &= \qt_{t}(s_t, a_t) - V_t\\
    &= \left(\qt_{t}(s_t, a_t) - Q_*(s_t, a_t)\right) + \left(Q_*(s_t, a_t) - V_t\right)\\
    &=\left(\qt_{t}(s_t, a_t) - Q_*(s_t, a_t)\right) + \gamma\zeta_{t} + \gamma\delta_{t + 1} - \gamma\phi_{t + 1}\\
    &\stackrel{(a)}{\leq}2\beta\sum_{\tau = 1}^t\gamma^{t - \tau}\norm*{\mathcal{K}_{(s_t, a_t)}}_{\wh_{\tau}}+ \frac{2\epsilon}{(1 - \gamma)^2}+ \gamma\zeta_{t} + \gamma\delta_{t + 1} - \gamma\phi_{t + 1}
\end{align*}
where (a) is due to the upper bound in Lemma~\ref{lem:est-lower-bound}. Taking a summation on both sides, we have,
\begin{align*}
    \sum_{t = 1}^T\delta_t 
    &\leq \frac{2\epsilon T}{(1 - \gamma)^2} + 2\beta\sum_{t = 1}^T\sum_{\tau = 1}^t \gamma^{t - \tau}\norm*{\mathcal{K}_{(s_t, a_t)}}_{\wh_{\tau}} + \gamma\sum_{t = 1}^T\left(\zeta_t + \delta_{t + 1} - \phi_{t + 1}\right)\\
    &\stackrel{(a)}{\leq} \frac{2\epsilon T}{(1 - \gamma)^2} +  2\beta\sqrt{\sum_{t = 1}^T\sum_{\tau = 1}^t\gamma^{t - \tau}}\sqrt{\sum_{t = 1}^T\sum_{\tau = 1}^t\gamma^{t - \tau}\norm*{\mathcal{K}_{(s_t, a_t)}}_{\wh_{\tau}}^2} + \gamma\sum_{t = 1}^T\left(\zeta_t + \delta_{t + 1} - \phi_{t + 1}\right)\\
    &\stackrel{(b)}{\leq} \frac{4\epsilon T}{(1 - \gamma)^2} +  2\beta\sqrt{\frac{ T d_{\lambda}/\lambda\cdot \ln\frac{e(T + \lambda)}{\lambda}}{\ln\left(1 + 1/\lambda\right)(1 - \gamma)^3}} + \gamma\sum_{t = 1}^T\left(\zeta_t + \delta_{t + 1}\right),
\end{align*}
where (a) is by Cauchy-Schwarz and (b) is due to Lemma~\ref{lem:sum-of-ucb-bound}, Lemma~\ref{lem:effective-dimension-bound}, and~\eqref{eq:phibound}. After rearranging the terms and using the fact $\abs{\delta_t}\leq 1/(1 - \gamma)$, we have
\begin{align*}
    \sum_{t = 1}^T\delta_t \leq
    \frac{4\epsilon T}{(1 - \gamma)^3} +  2\beta\sqrt{\frac{ T d_{\lambda}/\lambda\cdot \ln\frac{e(T + \lambda)}{\lambda}}{\ln\left(1 + 1/\lambda\right)(1 - \gamma)^5}} + \gamma\sum_{t = 1}^T\zeta_t + \frac{2}{(1 - \gamma)^2}.
\end{align*}
Now note that $\left\{\zeta_t\right\}_{t = 1}^T$ is a Martingale difference sequence where each element is $(\frac{\sigma}{1 - \gamma})$-sub-Gaussian, therefore with probability at least $1 - \delta / 2$, 
\begin{align*}
    \sum_{t = 1}^T\delta_t \leq
    \frac{4\epsilon T}{(1 - \gamma)^3} +  2\beta\sqrt{\frac{ T d_{\lambda}/\lambda\cdot \ln\frac{e(T + \lambda)}{\lambda}}{\ln\left(1 + 1/\lambda\right)(1 - \gamma)^5}} + \frac{\sigma\sqrt{2T\ln(2/p)}}{1 - \gamma} + \frac{2}{(1 - \gamma)^2}.
\end{align*}
Finally, note that
\begin{align*}
\regret(T)
& = (1 - \gamma)\sum_{t = 1}^T\Delta_t\\
& \stackrel{(a)}{\leq} \frac{1}{1 - \gamma}\cdot\sum_{t = 1}^T \delta_t + \frac{2\epsilon T}{1 - \gamma}\\
&\leq
    \frac{6\epsilon T}{(1 - \gamma)^2} +  2\beta\sqrt{\frac{ T d_{\lambda}/\lambda\cdot \ln\frac{e(T + \lambda)}{\lambda}}{\ln\left(1 + 1/\lambda\right)(1 - \gamma)^3}} + \frac{\sigma\sqrt{2T\ln(2/p)}}{1 - \gamma} + \frac{2}{1 - \gamma}\\
    &\leq \left(6\rho\sqrt{\lambda} + 2\epsilon \sqrt{Td_{\lambda}} +4\sigma \sqrt{
    d_{\lambda}\ln\frac{e(T + \lambda)}{\lambda} +  \ln\frac{2}{p}+c_{\lambda}}\right)\sqrt{\frac{ T d_{\lambda}/\lambda\cdot \ln\frac{e(T + \lambda)}{\lambda}}{\ln\left(1 + 1/\lambda\right)(1 - \gamma)^5}}\\
    &\ \ \ \  + \frac{6\epsilon T}{(1 - \gamma)^2} + \frac{\sigma\sqrt{2T\ln(2/p)}}{1 - \gamma} + \frac{2}{1 - \gamma}\\
    & = \mathcal{O}\left(\sqrt{\frac{Td_{\lambda}\log\frac{e(T + \lambda)}{\lambda}}{\log(1 + 1/\lambda)(1 - \gamma)^5}}\left(\rho + \epsilon\sqrt{\frac{d_{\lambda}T}{\lambda}} + \sigma\sqrt{\frac{d_{\lambda}\log\frac{e(T + \lambda)}{\lambda p} + c_{\lambda}}{\lambda}}\right)\right)\\
    &\ \ \ \ +\mathcal{O}\left(\frac{\sigma\sqrt{T\log(1/p)}}{1 - \gamma} + \frac{\epsilon T}{(1 - \gamma)^2} + \frac{1}{1 - \gamma}\right)\\
    & \stackrel{(b)}{=} \mathcal{O}\left(\sqrt{\frac{Td_{\lambda}\log\frac{e(T + \lambda)}{\lambda}}{\log(1 + 1/\lambda)(1 - \gamma)^5}}\left(\rho + \epsilon\sqrt{\frac{d_{\lambda}T}{\lambda}} + \sigma\sqrt{\frac{d_{\lambda}\log\frac{e(T + \lambda)}{\lambda p} + c_{\lambda}}{\lambda}}\right)\right),
\end{align*}
where in (a) we used $\Delta_t = \delta_t - \phi_t$ and~\eqref{eq:phibound}, in (b) we used the fact that $d_{\lambda} \geq 1$, $\rho \geq 1 - \gamma$, and
\begin{align*}
    \frac{\log\frac{e(K + \lambda)}{\lambda}}{\log(1 + 1/\lambda)}\geq 1,\ \ \ \ \ \ \ \ 
    \frac{\log\frac{e(K + \lambda)}{\lambda}}{\log(1 + 1/\lambda)\sqrt{\lambda}} \geq 1.
\end{align*}
This concludes the proof.
\end{proof}

\end{document}